\documentclass[twoside,11pt]{article}

\usepackage{jmlr2e}

\jmlrheading{1}{2013}{1-48}{4/00}{10/00}{Qiang Liu and Alexander Ihler}

\ShortHeadings{Variational algorithms for Marginal MAP}{Liu and Ihler}
\firstpageno{1}

\usepackage{algorithm}
\usepackage{algorithmic}

\usepackage{enumerate}
\usepackage{verbatim}
\usepackage{amsmath, amsfonts, amssymb}
\usepackage{mathrsfs}  %
\usepackage[colorlinks,citecolor=blue,lin  kcolor=blue]{hyperref}  
\usepackage{graphicx}
\usepackage{subfigure}
\usepackage{array}
\usepackage{wrapfig}

\usepackage{tikz}
\usetikzlibrary{%
	calc,%
	decorations.pathmorphing,%
	fadings,%
	shadings%
}

\newcommand{\figref}[1]{Fig.~\ref{#1}}
\newcommand{\X}{\mathcal{X}}

\newcommand{\E}{\mathbb{E}}

\newcommand{\margin}{b}
\newcommand{\cross}{\partial_{AB}}

\newcommand{\block}[1]{}
\newcommand{\epsilontt}{\lambda^t}

\newcommand{\qiangnew}[1]{#1}
\newcommand{\qiangold}[1]{#1}

\newcommand{\coloneq}{\mathrel{\mathop:}=}

\newcommand{\margpoly}{\mathbb{M}}

\newcommand{\localpoly}{\mathbb{L}} 
\newcommand{\N}{\mathbb{N}}
\newcommand{\neib}[1]{\partial_{#1}}
\newcommand{\M}{\mathbb{M}}
\newcommand{\R}{\mathbb{R}}

\newcommand{\Q}{\mathcal{Q}}

\newcommand{\bx}{\boldsymbol{x}}

\newcommand{\deltaIndic}{\boldsymbol{1}}

\newcommand{\vtheta}{\boldsymbol{\theta}}
\newcommand{\vtau}{{\boldsymbol{\tau}}}
\newcommand{\vx}{\boldsymbol{x}}

\newcommand{\vpsi}{\boldsymbol{\psi}}

\newcommand{\vrho}{{\boldsymbol{\rho}}}

\newcommand{\Hitau}{H_{{i}}(\vtau)}
\newcommand{\Iijtau}{I_{{ij}}(\vtau)}

\newcommand{\HAB}{H_{A|B}(\vtau)}
\newcommand{\HABtau}{H_{A|B}(\vtau)}
\newcommand{\HAtau}{H_{A}(\vtau)}
\newcommand{\HB}{H_{B}(\vtau)}
\newcommand{\HBtau}{H_{B}(\vtau)}
\newcommand{\Hxtau}{H(\vtau)}
\newcommand{\Hx}{H(\vtau)}
\newcommand{\hatHBtau}{\hat{H}_B(\vtau)} 
\newcommand{\hatHABtau}{\hat{H}_{A|B}(\vtau)} 
\newcommand{\hatHBeptau}{\hat{H}_B(\eptau)}
\newcommand{\HABq}{H_{A|B}( \boldsymbol{\qtau})}
\newcommand{\HABtauT}{H_{A|B}(\vtau ~;~ T)}
\newcommand{ \HAtauA}{H_{A}(\vtau_A)}
\newcommand{\HAkBktau}{H_{A_k|B_k}(\vtau)}

\newcommand{\qtau}{\tau}

\newcommand{\margpolyNull}{\margpoly^o}

\DeclareMathOperator*{\argmax}{arg\,max}
\DeclareMathOperator*{\argmin}{arg\,min}

\usepackage{amsthm}   
\newtheorem{thm}{Theorem}[]
\newtheorem{cor}[thm]{Corollary}
\newtheorem{lem}[thm]{Lemma} 
\newtheorem{exa}[thm]{Example} 
 
\newtheorem{pro}[thm]{Proposition}    
\newtheorem{mydef}{Definition}[section]

\begin{document}

\title{Variational Algorithms for Marginal MAP}

\author{\name Qiang Liu \email qliu1@uci.edu \\
       \addr Donald Bren School of Information and Computer Sciences\\
	University of California, Irvine\\
	Irvine, CA, 92697-3425, USA\\ 
       \AND
       \name Alexander Ihler  \email ihler@ics.uci.edu \\
       \addr Donald Bren School of Information and Computer Sciences\\
	University of California, Irvine\\
	Irvine, CA, 92697-3425, USA}

\editor{XXXXXXX}

\maketitle

\begin{abstract}
The marginal maximum {\it a posteriori} probability (MAP) estimation problem, which calculates the mode of the marginal posterior distribution of a subset of variables with the remaining variables marginalized, is an important inference problem in many models, such as those with hidden variables or uncertain parameters. 
Unfortunately, marginal MAP can be NP-hard even on trees, and has attracted less attention in the literature compared to the joint MAP (maximization) and marginalization problems. 
We derive a general dual representation for marginal MAP that naturally integrates the marginalization and maximization operations into a joint variational optimization  problem, making it possible to easily extend most or all variational-based algorithms to marginal MAP. 
In particular, we derive a set of  ``mixed-product" message passing algorithms for marginal MAP, whose form is a hybrid of max-product, sum-product and a novel ``argmax-product" message updates. 
We also derive a class of convergent algorithms based on proximal point methods, including one that transforms the marginal MAP problem into a sequence of standard marginalization problems. 
Theoretically, we provide guarantees under which our algorithms give globally or locally optimal solutions, and provide novel upper bounds on the optimal objectives. 
Empirically, we demonstrate that our algorithms significantly outperform the existing approaches, including a state-of-the-art algorithm based on local search methods. 
\end{abstract}


\section{Introduction}
Graphical models such as Bayesian networks and Markov random fields provide a powerful framework for reasoning about
conditional dependency structures over many variables, and have found wide application in many areas including error correcting codes, computer vision, and computational biology \citep{Wainwright08,Koller_book}. 
Given a graphical model, which may be estimated from empirical data or constructed by domain expertise, 
the term \emph{inference} refers generically to answering
probabilistic queries about the model, 
such as computing marginal probabilities or maximum {\it a posteriori} estimates. 
Although these inference tasks are NP-hard in the worst case, 
recent algorithmic advances, including the development of variational methods and the family of algorithms collectively called belief propagation, provide approximate
or exact solutions for these problems in many practical circumstances.

In this work we will focus on three common types of inference tasks.
The first involves \emph{maximization} or \emph{max-inference}
tasks, sometimes called maximum {\it a posteriori} (MAP) or most probable explanation
(MPE) tasks, which look for a mode of the joint probability. 
The second are \emph{sum-inference} tasks, which include calculating the marginal probabilities
or the normalization constant of the distribution
(corresponding to the probability of evidence in a Bayesian network).  Finally, the main focus of
this work is on \emph{marginal MAP}, a type of \emph{mixed-inference} problem that seeks a partial configuration
of variables that maximizes those variables' marginal probability, with the remaining variables summed out.%
\footnote{In some literature \citep[e.g.,][]{Park04}, marginal MAP is simply referred to as MAP, and the joint MAP problem is called MPE.}
\qiangnew{Marginal MAP plays an essential role in many practical scenarios where there exist hidden variables or uncertain parameters. 
For example, a marginal MAP problem can arise as a MAP problem on models with hidden variables whose predictions are not of interest, 
or as a robust optimization variant of MAP with some unknown or noisily observed parameters marginalized w.r.t.\ a prior distribution. }
It can be also treated as a special case of the more complicated frameworks of stochastic programming \citep{birge1997introduction} or decision networks 
 \citep{howard2005influence, liu12b}.

These three types of inference tasks are listed in order of increasing difficulty: max-inference is NP-complete, while sum-inference is \#P-complete,
and mixed-inference is $\mathrm{NP}^{\mathrm{PP}}$-complete \citep{Park04, de2011new}. 
Practically speaking, max-inference tasks have a host of efficient algorithms such as
loopy max-product BP, tree-reweighted BP, and dual decomposition~\cite[see e.g., ][]{Koller_book, Sontag_optbook}.
Sum-inference is more difficult than max-inference:  
for example there are models, such as those with binary attractive pairwise potentials, on which sum-inference is \#P-complete but max-inference is tractable \citep{greig1989exact, jerrum1993polynomial}. 

Mixed-inference is even much harder than either max- or sum- inference
problems alone: marginal MAP can be NP-hard even on tree structured graphs, as illustrated in the example in \figref{fig:hiddenchain} \citep{Koller_book}. 
The difficulty arises in part because the
max and sum operators do not commute, causing the feasible elimination orders to have much higher induced width than for sum- or max-inference. 
Viewed another way, the marginalization step may destroy the dependency structure of the original graphical model, making the subsequent maximization step far more challenging. 
Probably for these reasons, there is much less work on marginal MAP than that on joint MAP or marginalization, despite its importance to many practical problems. 
\qiangnew{In practice, it is common to over-use the simpler joint MAP or marginalization even when marginal MAP would be more appropriate. This may cause serious problems, as we illustrate in Example~\ref{exa:weather_dilemma} and our empirical results in Section~\ref{sec:experiments}. }

\textbf{Contributions.} 
We reformulate the mixed-inference problem to a joint maximization problem as a free energy objective that extends the well-known log-partition function duality form, making it possible to easily extend essentially arbitrary variational algorithms to marginal MAP. 
In particular, we propose a novel ``mixed-product" BP algorithm that is a hybrid of max-product, sum-product, and a special ``argmax-product" message updates, as well as a convergent proximal point algorithm that works by iteratively solving pure (or annealed) marginalization tasks. We also present junction graph BP variants of our algorithms, that work on models with higher order cliques.  We also discuss mean field methods and highlight their connection to the expectation-maximization (EM) algorithm. 
We give theoretical guarantees on the global and local optimality of our algorithms for cases when the sum variables form tree structured subgraphs. Our numerical experiments show that our methods can provide significantly better solutions than existing algorithms, including a similar hybrid message passing algorithm by \citet{Jiang10} and a state-of-the-art algorithm based on local search methods. 
\qiangold{A preliminary version of this work has appeared in \citet{liu11marginal_MAP}.} 
 
\textbf{Related Work.} 
Expectation-maximization (EM) or variational EM provide one straightforward approach for marginal MAP, by viewing the sum nodes as hidden variables and the max nodes as parameters to be estimated; 
however, EM is prone to getting stuck at sub-optimal configurations. 
The classical state-of-the-art approaches include local search methods \citep[e.g.,][]{Park04},  Markov chain Monte Carlo methods \citep[e.g.,][]{Doucet02, yuan2004annealed}, and variational elimination based methods \citep[e.g.,][]{dechter2003mini, maua2012anytime}. 
\citet{Jiang10} recently proposed a hybrid message passing algorithm that has a similar form to our mixed-product BP algorithm, but without theoretical guarantees; we show in Section~\ref{sec:compare_jiang} that \citet{Jiang10} can be viewed as an approximation of the marginal MAP problem that exchanges the order of sum and max operators. 
Another message-passing-style algorithm was proposed very recently in \citet{altarelli2011stochastic} for general multi-stage stochastic optimization problems based on survey propagation, which again does not have optimality guarantees and has a relatively more complicated form. 
Finally, \citet{ibrahimi2011robust} introduces a robust max-product belief propagation for solving a \qiangold{related} 
worst-case robust optimization problem, where the hidden variables are minimized instead of marginalized. 
To the best of our knowledge, our work is the first general variational framework for marginal MAP, and provides the first strong optimality guarantees. 

We begin in Section~\ref{sec:background} by introducing background on graphical models and variational inference.  
We then introduce a novel variational dual representation for marginal MAP in Section~\ref{sec:mixduality}, and propose analogues of the Bethe and tree-reweighted approximations in Section~\ref{sec:variational}. 
A class of ``mixed-product" message passing algorithms is proposed and analyzed in Section~\ref{sec:message} and convergent alternatives are proposed in Section~\ref{sec:proximal} based on proximal point methods. 
We then discuss the EM algorithm and its connection to our framework in Section~\ref{sec:EM}, 
and extend our algorithms to junction graphs in Section~\ref{sec:junctiongraph}. 
%
Finally, we present numerical results in Section~\ref{sec:experiments} and conclude the paper in Section~\ref{sec:conclusion}. 

\section{Background}
\label{sec:background}
\subsection{Graphical Models}
\newcommand{\I}{\mathcal{I}}
Let $\bx = \{x_1,  x_2, \cdots, x_n\}$ be a random vector in a discrete space $\X= \X_1\times \cdots \times \X_n$. Let $V = \{ 1, \cdots, n \}$. For an index set $\alpha \subseteq V$, denote by $\vx_\alpha$ the sub-vector $\{x_{i}  \colon   i \in \alpha \}$, and similarly, $\X_{\alpha}$ the cross product of $\{\X_i \colon i \in \alpha \}$. 
A graphical model defines a factorized probability on $\bx$, 
\begin{align}
p(\bx) = \frac{1}{Z(\vpsi)} \prod_{\alpha \in \I} \psi_{\alpha}( \vx_\alpha) &&\text{~~~~or~~~~}  &&  p(\bx ; \vtheta) = \exp[\sum_{\alpha \in \I}  \theta_{\alpha}(\vx_\alpha)  - \Phi(\vtheta)],  
\end{align}
%
where $\I$ is a set of subsets of variable indexes, $\psi_{\alpha} \colon \X_{\alpha} \to \R^+ $ is called a factor function,
and $\theta_{\alpha}(\vx_\alpha) = \log \psi_{\alpha}(\vx_\alpha)$.
Since the $x_i$ are discrete, the functions $\psi$ and $\theta$ are tables; by alternatively viewing $\theta$ as
a vector, it is interpreted as the natural parameter in an overcomplete, exponential family representation.
Let $\vpsi$ and $\vtheta$ be the joint vector of all $\psi_{\alpha}$ and $\theta_{\alpha}$ respectively, e.g., $\vtheta = \{ \theta_{\alpha}(\vx_\alpha)  \colon \alpha \in I, \vx_\alpha \in \X_{\alpha} \} $. 
The normalization constant $Z(\vpsi)$, called \emph{partition function}, normalizes the probability to sum to one, and $\Phi(\vtheta) \coloneq \log Z(\vpsi)$ is called the log-partition function, 
\begin{align*}
\Phi(\vtheta) =  \log \sum_{\vx \in \X}\exp[\theta(\vx)],
\end{align*}
where we define $\theta(\vx) = \sum_{\alpha \in \I} \theta_{\alpha}(\vx_\alpha)$ to be the joint  potential function that maps from $\X$ to $\R$.  
%
%
%
The factorization structure of $p(\vx)$ can be represented by an undirected graph 
$G=(V,E)$, where each node $i\in V$ maps to a variable $x_i$, and each edge $(ij)\in E$ corresponds to two variables $x_i$ and $x_j$ that coappear in some factor function $\psi_{\alpha}$, that is, $\{ i, j \} \subseteq \alpha$. 
The set $\mathcal{I}$ is then a set of cliques (fully connected subgraphs) of $G$. 
For the purpose of illustration, we mainly restrict our scope on the set of pairwise models, on which $\I$ is the set of nodes and edges, i.e., $\I = E \cup V$.  However, we show how to extend our algorithms to models with higher order cliques in Section~\ref{sec:junctiongraph}. 


\subsection{Sum-Inference Problems and Variational Approximation}
Sum-inference is the task of marginalizing (summing out) variables in the model, e.g., calculating the marginal probabilities of single variables, or the normalization constant $Z$, 
\begin{align}
p(x_{i}) =  \sum_{\vx_{V \setminus \{i\}}} \exp[\theta(\vx) - \Phi(\vtheta)],  &&   \Phi(\vtheta) = \log \sum_{\vx} \exp[\theta(\vx)].
\end{align}
Unfortunately, the problem is generally \#P-complete, and the straightforward calculation requires summing over an exponential number of terms. Variational methods are a class of approximation algorithms that transform the marginalization problem into a continuous optimization problem, which is then typically solved approximately. 

\textbf{Marginal Polytope.}  
The marginal polytope is a key concept in variational inference. We define the \emph{marginal polytope} $\margpoly$ to be the set of local marginal probabilities $\vtau = \{\tau_{\alpha}(\vx_\alpha) \colon \alpha \in \I \}$ that are extensible to a valid joint distribution, i.e., 
\begin{equation}
\margpoly = \{\vtau \ :\  \text{$\exists$ joint distribution $q(\vx)$, s.t. $\tau_{\alpha}(\vx_\alpha) = \sum_{\vx_{V \setminus \alpha}} q(\vx)$ for $\forall \alpha \in \I$}\}.
\label{equ:marginalpolytope}
\end{equation}
Denote by $\Q[\vtau]$ the set of joint distributions whose marginals are consistent with $\vtau \in \margpoly$; by the principle of maximum entropy \citep{maxent}, there exists a unique distribution in $\Q[\vtau]$ that has  maximum entropy and follows the exponential family form for some $\vtheta$.%
\footnote{In the case that $p(\vx)$ has zero elements, the maximum entropy distribution is still unique and satisfies the exponential family form, but the corresponding $\vtheta$ has negative infinite values \citep{maxent}.} 
With an abuse of notation, we denote these unique global distributions by $\tau(\vx)$, and we do not distinguish $\tau(\vx)$ and $\vtau$ when it is clear from the context. 

\textbf{Log-partition Function Duality.} 
A key result to many variational methods is that the log-partition function $\Phi(\vtheta)$ is a convex function of $\vtheta$ and can be rewritten into a convex dual form, 
\begin{align}
\Phi(\vtheta) = \max_{\vtau \in \margpoly}  \big\{ \langle \vtheta, \vtau \rangle + \Hxtau  \big\},
\label{equ:sumduality}
\end{align}
where $\langle \vtheta, \vtau \rangle = \sum_{\alpha}  \sum_{\vx_\alpha} \theta_{\alpha}(\vx_\alpha) \tau_\alpha(\vx_\alpha)$ 
is the vectorized inner product, 
and $\Hxtau$ is the entropy of the corresponding global distribution $\tau(\vx)$, i.e., $\Hxtau = - \sum_{\vx} \tau(\vx) \log \tau(\vx)$. The unique maximum $\vtau^*$ of \eqref{equ:sumduality} exactly equals the marginals of the original distribution $p(\vx; \vtheta)$, that is, $\tau^*(\vx) = p(\vx; \vtheta)$. 
We call $F_{sum}(\vtau, \vtheta) = \langle \vtheta, \vtau \rangle + \Hxtau$ the sum-inference free 
energy (although technically the {\it negative} free energy). 

The dual form \eqref{equ:sumduality} transforms the marginalization problem into a continuous optimization, but does not make it any easier: 
the marginal polytope $\margpoly$ is defined by an exponential number of linear constraints, and the entropy term in the objective function is as difficult to calculate as the log-partition function. 
However, \eqref{equ:sumduality}
provides a framework for deriving efficient approximate inference algorithms by approximating
both the marginal polytope and the entropy \citep{Wainwright08}. 

\textbf{BP-like Methods.} Many approximation methods replace $\margpoly$ with the \emph{locally consistent polytope} $\localpoly$; in pairwise models, it is the set of singleton and pairwise ``pseduo-marginals" $\{\tau_i (x_i)\colon i \in
V\}$ and $\{\tau_{ij}(x_i, x_j) \colon (ij) \in E\}$ that are consistent on their intersections, i.e., 
\begin{equation}\label{equ:localpolydefine}
\localpoly =  \{ \tau_i, \tau_{ij} ~ \colon ~ \sum_{x_i} \tau_{ij}(x_i,x_j) = \tau_j(x_j), ~ 
\sum_{x_i}\tau_{i}(x_i) = 1, ~ \tau_{ij}(x_i, x_j) \geq 0 \}.
\end{equation}
Since not all such pseudo-marginals have valid global distributions, 
it is easy to see that $\localpoly$ is an outer bound of $\margpoly$, that is, $ \margpoly  \subseteq \localpoly$. 
\qiangold{Note that this means there may not exist a global distribution $\tau(\vx)$ for $\vtau$ in $\localpoly$. }

The free energy remains intractable (and is not even well-defined) in $\localpoly$.  We
typically approximate the free energy by a combination of singleton
and pairwise entropies, which only requires knowing $\tau_i$ and $\tau_{ij}$. For
example, the Bethe free energy approximation \citep{yedidia2003understanding} is 
\begin{align}
\Hxtau \approx \sum_{i\in V} \Hitau - \sum_{(ij) \in E} \Iijtau, &&
\Phi(\vtheta) \approx \max_{\vtau \in \localpoly} \big \{ \langle \vtheta, \vtau \rangle + \sum_{i\in V} \Hitau - \sum_{(ij) \in E} \Iijtau  \big \} ,
\label{equ:bethe}
\end{align}
where $\Hitau$ is the entropy of $\tau_i(x_i)$ and $\Iijtau$ the mutual information of $x_i$ and $x_j$, i.e., 
\begin{align*}
\Hitau  =  - \sum_{x_i} \tau_i(x_i) \log \tau_i(x_i), &&  \Iijtau = \sum_{x_i,x_j} \tau_{ij}(x_i, x_j) \log \frac{\tau_{ij}(x_i, x_j)}{\tau_i(x_i) \tau_j(x_j)}.
\end{align*}
We sometimes abbreviate $\Hitau$ and $\Iijtau$ into $H_i$ and $I_{ij}$ for convenience.   
The well-known loopy belief propagation (BP) algorithm of \citet{pearl1988probabilistic} can be interpreted as a fixed point algorithm to optimize the Bethe free energy in \eqref{equ:bethe} on the locally consistent polytope $\localpoly$ \citep{yedidia2003understanding}. 
Unfortunately, the Bethe free energy is a non-concave function of $\vtau$, causing \eqref{equ:bethe} to be a non-convex optimization. The tree reweighted (TRW) free energy is a convex surrogate of the Bethe free energy \citep{Wainwright_TRBP}, 
\begin{align}
\Phi(\vtheta) \approx     \max_{\vtau \in \localpoly} \big \{  \langle \vtheta, \vtau \rangle + \sum_{i\in V} \Hitau - \sum_{(ij) \in E} \rho_{ij} \Iijtau  \big \},
\label{equ:TRW}
\end{align}
where $\{\rho_{ij}  \colon (ij) \in E \}$ is a set of positive edge appearance probabilities obtained from
a weighted collection of spanning trees of $G$ (see \citet{Wainwright_TRBP} and Section~\ref{sec:trw_bp_marginalMAP} for the detailed definition).
The TRW approximation in \eqref{equ:TRW} is a convex optimization problem, and is guaranteed to give an upper bound of the true log-partition function. 
A message passing algorithm similar to loopy BP, called tree reweighted BP, can be derived as a fixed point algorithm for solving the convex optimization in \eqref{equ:TRW}. 

\textbf{Mean-field-based Methods.} Mean-field-based methods are another set of approximate inference algorithms, which work by restricting $\margpoly$ to a set of tractable distributions, on which both the marginal polytope and the joint entropy are tractable. 
Precisely, let $\margpoly_{mf}$ be a subset of $\margpoly$ that corresponds to a set of tractable distributions, e.g., the set of fully factored distributions, $\margpoly_{mf} = \{\vtau \in \margpoly \colon \tau(\vx) = \prod_{i\in V} \tau_i(x_i) \}$. 
\qiangold{Note that the joint entropy $\Hxtau$ for any $\vtau \in \margpoly_{mf}$ decomposes to the sum of singleton entropies $\Hitau$ of the marginal distributions $\tau_i(x_i)$. 
This method then approximates the log-partition function \eqref{equ:sumduality} by}
\begin{align}
\max_{\vtau \in \margpoly_{mf}} \big \{  \langle \vtheta, \vtau \rangle +  \sum_{i\in V}\Hitau \big \} , \label{equ:meanfield}
\end{align}
which is guaranteed to give a lower bound of the log-partition function. Unfortunately, mean field methods usually lead to non-convex optimization problems, because $\margpoly_{mf}$ is often a non-convex set. In practice, block coordinate descent methods can be adopted to find the local optima of \eqref{equ:meanfield}. 

\subsection{Max-Inference Problems}
Combinatorial maximization (max-inference), or maximum \emph{a posteriori} (MAP), problems are the tasks of finding a mode of the joint probability.  That is, 
\begin{align}
\Phi_{\infty}(\vtheta) =  \max_{\vx}  \theta(\vx) ,   ~~~~~~~~~~~  \vx^* = \argmax_{\vx}  \theta(\vx), 
\label{equ:max_inference}
\end{align}
where $\vx^*$ is a MAP configuration and $\Phi_{\infty}(\vtheta)$ the optimal energy value.  
This problem can be reformed into a linear program, 
\begin{equation}
\Phi_{\infty}(\vtheta) =  \max_{\vtau \in \M} \langle \vtheta, \vtau \rangle,
\label{equ:maxduality}
\end{equation}
which attains its maximum when $\tau^*(\vx) = \deltaIndic(\vx = \vx^*)$, 
where $\deltaIndic(\cdot)$ is the Kronecker delta function, defined as $\deltaIndic(t) = 1$ if 
condition $t$ is true, and zero otherwise.
If there are multiple MAP solutions, say $\{\vx^{*k} \colon k=1,\ldots,K \}$, then any convex combination $\sum_k c_k \deltaIndic(\vx = \vx^{*k})$ with $\sum_k c_k =1, c_i\geq0$ leads to a maximum of \eqref{equ:maxduality}. 

\qiangold{The problem in \eqref{equ:maxduality} remains NP-hard, because the marginal polytope $\margpoly$ includes exponentially many inequality constraints. }
Most variational methods for MAP \citep[e.g.,][]{wainwright2005map, werner2007linear} 
can be interpreted as relaxing $\margpoly$ to the locally consistent polytop $\localpoly$, yielding a linear
relaxation of the original integer programming problem.  Note that
\eqref{equ:maxduality} differs from \eqref{equ:sumduality} only by its lack of an
entropy term; in the next section, we generalize this similarity to marginal
MAP. 


\subsection{Marginal MAP Problems}
%
Marginal MAP is simply a hybrid of the max- and sum- inference tasks. Let $A$ be a subset of nodes $V$, and $B = V\backslash A$ be the complement of $A$. The
marginal MAP problem seeks a partial configuration $\vx_B^*$ that has the maximum marginal probability $p(\vx_B) = \sum_{\vx_{A}} p(\vx)$, where $A$ is the set of {sum} nodes to be marginalized out, and $B$ the {max} nodes to be optimized. 
We call this a type of ``mixed-inference'' problem, since it involves more than one type of variable elimination operator.
\qiangold{To facilitate developing our duality results, we formulate marginal MAP in terms of the exponential family representation, }
\begin{align}
\Phi_{AB}(\vtheta) = \max_{\vx_B} Q(\vx_B; \vtheta), &&  \text{where~~~} Q(\vx_B; \vtheta) =  \log \sum_{\vx_A} \exp[\theta(\vx)], 
\label{equ:marginalMAP}
\end{align}
\qiangold{where the maximum point $\vx_B^*$ of $Q(\vx_B; \vtheta)$ is the marginal MAP solution. }
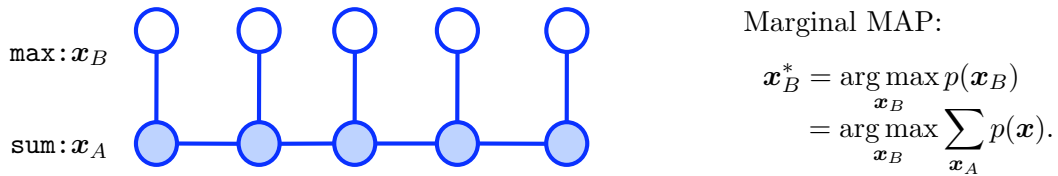
\begin{figure}[tb] \centering
\begin{picture}(0,0)
\thicklines \setlength{\unitlength}{1.3cm}
\put(-5, 1.1){\tt max:$\vx_B$}
\put(-5, .15){\tt sum:$\vx_A$}
\put(2.5, 1.4){Marginal MAP:}
\put(2.7, .85){$\displaystyle \vx_B^* = \argmax_{\vx_B} p(\vx_B)$}
\put(2.97, .32){$\displaystyle \ \ = \argmax_{\vx_B} \sum_{\vx_A} p(\vx)$.}
\end{picture}
\hspace{-5cm}\includegraphics[width=.4\columnwidth]{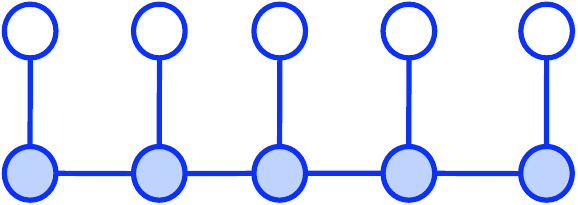} 
\caption{An example from \citet{Koller_book} in which a marginal MAP query on a tree requires exponential time complexity. 
The marginalization over $\vx_A$ destroys the conditional dependency structure in the marginal distribution $p(\vx_B)$, causing an intractable maximization problem over $\vx_B$. The exact variable elimination method, \qiangold{which sequentially marginalizes the sum nodes and then maximizes the max nodes,} has time complexity of $O(\exp(n))$, where $n$ is the length of the chain. 
}
\label{fig:hiddenchain}
\end{figure}
Although similar to max- and sum-inference, marginal
MAP is significantly harder than either of them. 
A classic example is shown in \figref{fig:hiddenchain}, where marginal MAP is NP-hard even on a tree structured graph \citep{Koller_book}.  The main difficulty arises 
because the {max} and {sum} operators do not commute, which 
restricts feasible elimination orders to those with \emph{all} the sum nodes eliminated
before \emph{any} max nodes. 
In the worst case, marginalizing the sum nodes $\vx_A$ may destroy any conditional independence among the max nodes $\vx_B$, 
making it difficult to represent or optimize $Q(\vx_B;\theta)$, 
even when the sum part alone is tractable 
(such as when the nodes in $A$ form a tree). 

Despite its computational difficulty, marginal MAP plays an essential role in many practical scenarios. The marginal MAP configuration $\vx_B^*$ in \eqref{equ:marginalMAP} is Bayes optimal in the sense that it minimizes the expected error on $B$, $\E[\deltaIndic(\vx_B^* = \vx_B)]$, where $\E[\cdot]$ denotes the expectation under distribution $p(\vx ; \vtheta)$.  Here, the variables $\vx_A$ are not included in the error criterion, for example because they are ``nuisance" hidden variables of no direct interest, or unobserved or inaccurately measured model parameters.  In contrast, the joint MAP configuration $\vx^*$ minimizes the joint error $\E[\deltaIndic(\vx^* = \vx)]$, but gives no guarantees on the partial error $\E[\deltaIndic(\vx_B^* = \vx_B)]$. In practice, perhaps because of the wide availability of efficient algorithms for joint MAP, researchers tend to over-use joint MAP even in cases where marginal MAP would be more appropriate. 
The following toy example shows that this seemingly reasonable approach can sometimes cause serious problems.  

\begin{exa}[Weather Dilemma]
\label{exa:weather_dilemma}
Denote by $x_b \in \{ {\tt {\tt rainy}}, {\tt {\tt sunny}} \}$ the weather condition of Irvine, and $x_a \in \{ {\tt walk}, {\tt drive}\}$ whether Alice drives or walks to the school depending on the weather condition. Assume the probabilities of $x_b$ and $x_a$ are 
\\
\begin{tabular}{p{5cm} p{7cm}}
 \vspace{-.5\baselineskip} 
 \begin{tabular}{p{.8cm} p{3cm}}
 \vspace{0pt}  $p(x_b):$ &  \vspace{0pt}  \begin{tabular}{|c|c|} \hline {\tt rainy} & $0.4$ \\ \hline {\tt sunny} & $0.6$ \\ \hline  \end{tabular} 
\end{tabular}
 & 
\vspace{-.5\baselineskip}
 \begin{tabular}{p{1.2cm} p{3cm}}
 \vspace{0pt}  $p(x_a | x_b):$ &   \vspace{0pt}  \begin{tabular}{|c|c|c|}\hline  & {\tt walk} & {\tt drive} \\  \hline {\tt rainy} & $1/8$ & $7/8$ \\ \hline {\tt sunny} & $1/2$ & $1/2$ \\ \hline  \end{tabular}
\end{tabular}
 \vspace{.8\baselineskip}
\end{tabular} \\
The task is to calculate the most likely weather condition of Irvine, which is obviously {\tt sunny} according to $p(x_b)$. The marginal MAP, $x_b^* = \argmax_{x_b} p(x_b) = {\tt sunny}$, gives the correct answer. However, the full MAP estimator, $[x_a^*, x_b^*] = \argmax p(x_a, x_b) = [ {\tt drive}, {\tt rainy}]$, gives answer $x_b^* = {\tt rainy}$ (by dropping the $x_a^*$ component), which is obviously wrong. Paradoxically, if $p(x_a | x_b)$ is changed (say, corresponding to a different person), the solution returned by full MAP could be different. 
\end{exa}
In the above example, since no evidence on $x_a$ is observed, the conditional probability $p(x_a | x_b)$ does not provide useful information for $x_b$, 
but instead provides misleading information when it is incorporated in the full MAP estimator. The marginal MAP, on the other hand, eliminates the influence of the irrelevant $p(x_a  | x_b)$ by marginalizing (or averaging) $x_a$. In general, the marginal MAP and full MAP can differ significantly
when the uncertainty in the hidden variables changes as a function of $\vx_B$. 




\section{A Dual Representation for Marginal MAP}
\label{sec:mixduality}
In this section, we present our main result, a dual representation of the
marginal MAP problem \eqref{equ:marginalMAP}. Our dual
representation generalizes that of sum-inference in \eqref{equ:sumduality} and
max-inference in \eqref{equ:maxduality}, and provides a unified framework for
solving marginal MAP problems. 
\begin{thm}
\label{thm:duality}
The marginal MAP energy $\Phi_{AB}(\vtheta)$ in \eqref{equ:marginalMAP} has a dual representation, 
\begin{equation}
\Phi_{AB}(\vtheta) = \max_{\vtau \in \margpoly} \{  \langle \vtheta, \vtau \rangle + \HABtau \},
\label{equ:mixduality}
\end{equation}
where $\HABtau$ is a conditional entropy, 
$\HABtau = -\sum_{\vx} \qtau(\vx) \log \qtau(\vx_A | \vx_B)$. 
If $Q(\vx_B; \vtheta)$ has a unique maximum $\vx_B^*$, the maximum point $\vtau^*$ of  \eqref{equ:mixduality} is also unique,  satisfying $\tau^*(\vx) = \qtau^*(\vx_B) \qtau^*(\vx_A | \vx_B)$, where
$\qtau^*(\vx_B) = \deltaIndic(\vx_B = \vx_B^*)$ and $\qtau^*(\vx_A | \vx_B) = p(\vx_A | \vx_B; \vtheta)$ 
\footnote{Since $\qtau(\vx_B)=0$ if $\vx_B \neq \vx_B^*$, we do not necessarily need to define $\qtau^*(\vx_A | \vx_B)$ for $\vx_B \neq \vx_B^*$.}. 
%
\end{thm}
\begin{proof}
\newcommand{\qtmp}{\qtau}
For any $\vtau \in \margpoly$ and its corresponding global distribution $\qtau(\vx)$, consider the conditional KL divergence between $\qtau(\vx_A | \vx_B)$ and $p(\vx_A |\vx_B ; \vtheta)$, 
\begin{align}
&D_\mathrm{KL}[ \qtmp(\vx_A| \vx_B) || p(\vx_A| \vx_B; \vtheta) ]
 = \sum_{\vx} \qtmp(\vx) \log \frac{\qtmp(\vx_A | \vx_B)}{p(\vx_A |\vx_B ; \vtheta)} \notag \\
&\qquad\qquad = -\HABq  - \E_\qtmp [\log p(\vx_A | \vx_B ; \vtheta)] \notag \\
&\qquad\qquad = -\HABq  -  \E_\qtmp [\theta(\vx)]  + \E_\qtmp[ Q(\vx_B; \vtheta)] 
\quad \geq \quad 0, \notag
\end{align}
where $\HABq$ is the conditional entropy on $\qtmp(\vx)$; 
the equality on the last line holds because $p(\vx_A | \vx_B ; \vtheta) = \exp(\theta(\vx) - Q(\vx_B; \vtheta))$; 
the last inequality follows from the nonnegativity of KL divergence, 
and is tight if and only if $\qtmp(\vx_A| \vx_B) = p(\vx_A| \vx_B; \vtheta)$ for all $\vx_A$ and $\vx_B$ that $\qtmp(\vx_B) \neq 0$.
Therefore, we have for any $\qtmp(\vx)$, 
\begin{equation*}
\Phi_{AB}(\vtheta) = \max_{\vx_B} Q(\vx_B ; \vtheta)  \geq \E_\qtmp[Q(\vx_B; \vtheta)]  \geq \E_\qtmp[\theta(\vx)] +\HABq.
\end{equation*}
It is easy to show that the two inequality signs are tight if and only if $\qtmp(\vx)$ equals $\qtau^*(\vx)$ as defined above.
Substituting $\E_\qtmp[\theta(\vx)]=\langle \vtheta,\vtau\rangle$ completes the proof.
\end{proof}
 \begin{table}[tb] \centering
\setlength{\extrarowheight}{5pt}
\begin{tabular}{ | l | l | l |}
\hline
Problem Type & Primal Form & Dual Form \\ \hline
Max-Inference & $\displaystyle  \log \max_{\vx} \exp(\theta(\vx))$ & $\displaystyle \max_{\vtau \in \margpoly} \{  \langle \vtheta, \vtau \rangle \}$ \\ \hline
Sum-Inference & $\displaystyle \log  \sum_{\vx} \exp(\theta(\vx))$ & $ \displaystyle \max_{\vtau \in \margpoly}  \{ \langle \vtheta, \vtau \rangle + \Hx \}$ \\  \hline
\textcolor{blue}{Marginal MAP} & \textcolor{blue}{$\displaystyle \log \max_{\vx_B}\sum_{\vx_A} \exp(\theta(\vx))$} & \textcolor{blue}{$\displaystyle \max_{\vtau \in \margpoly} \{  \langle \vtheta, \vtau \rangle+ \HAB  \} $} \\\hline  
\end{tabular}
\caption{The primal and dual forms of the three inference types. The dual forms of sum-inference and max-inference are well known; the form for marginal MAP is a contribution of this work. Intuitively, the max vs.\ sum operators in the primal form determine the conditioning set of the conditional entropy term in the dual form. }
\label{tab:threetasks}
\end{table}
\textbf{Remark 1.} If $Q(\vx_B ; \vtheta)$ has multiple maxima $\{ \vx^{*k}_B \}$, each corresponding to a distribution $\qtau^{*k}(\vx) = \deltaIndic(\vx_B = \vx_B^*) p(\vx_A | \vx_B; \vtheta)$, then 
 the set of maximum points of $\eqref{equ:mixduality}$ is the convex hull of $\{ \vtau^{*k} \}$. 

\textbf{Remark 2.} 
Theorem \ref{thm:duality} naturally integrates the marginalization and maximization sub-problems into one joint optimization problem, 
providing a novel and efficient treatment for marginal MAP beyond the traditional approaches that treat the marginalization sub-problem as a sub-routine of the maximization problem. As we show in Section~\ref{sec:message}, this enables us to derive efficient ``mixed-product" message passing algorithms that simultaneously takes marginalization and maximization steps, avoiding expensive and possibly wasteful inner loop steps in the marginalization sub-routine.

\textbf{Remark 3.} 
Since we have $\HAB = \Hx - \HB$ by the entropic chain rule \citep{InformationTheory}, 
the objective function in \eqref{equ:mixduality} can be view as a  ``truncated" free energy, 
\begin{align*}
F_{mix}(\vtau, \vtheta) \coloneq \langle \vtheta, \vtau \rangle + \HAB 
= F_{sum} (\vtau, \vtheta) - \HB, 
\end{align*}
where the entropy $\HB$ of the {max} nodes $\vx_B$ are removed from the regular 
sum-inference free energy $F_{sum}(\vtau , \vtheta) = \langle \vtheta, \vtau \rangle + \Hx$. Theorem \ref{thm:duality} generalizes the dual form of 
both sum-inference  \eqref{equ:sumduality} and max-inference \eqref{equ:maxduality}, since it reduces to those forms when 
the {max} set $B$ is empty or all nodes, respectively. 
Table~\ref{tab:threetasks} shows all three forms together for comparision. Intuitively, since the entropy $\HB$ is removed
from the objective, the optimal marginal $\qtau^*(\vx_B)$ tends to have lower
entropy and its probability mass concentrates on the optimal configurations $\{\vx_B^*\}$. 
%
Alternatively, the $\tau^*(\vx)$ can be interpreted as the marginals 
obtained by clamping the value of $\vx_B$ at
$\vx_B^*$ on the distribution $p(x; \vtheta)$, i.e., $\qtau^*(\vx) =
p(\vx | \vx_B = \vx_B^*; \vtheta)$.

\textbf{Remark 4.} 
Unfortunately, subtracting the $\HBtau$ term causes some subtle difficulties. First, $\HBtau$ (and hence $F_{mix}(\vtau, \vtheta)$) may be intractable to calculate even when the joint entropy $\Hxtau$ is tractable, because the marginal distribution $p(\vx_B) = \sum_{\vx_A} p(\vx)$ does not necessarily inherit the conditional dependency structure of the joint distribution. Therefore, the dual optimization in \eqref{equ:mixduality} may be intractable even on a tree, reflecting the intrinsic difficulty of marginal MAP compared to full MAP or marginalization. 
Interestingly, we show in the sequel that a certificate of 
optimality can still be obtained on general tree graphs in some cases.  

Secondly, the conditional entropy $\HABtau$ (and hence $F_{mix}(\vtau, \vtheta)$) is concave, but not strictly concave, with respect to $\vtau$.
This creates additional difficulty when optimizing \eqref{equ:mixduality}, since
many iterative optimization algorithms, such as coordinate descent, can lose their typical convergence or optimality guarantees when the objective function is not strongly convex. 

\textbf{Smoothed Approximation.} 
To sidestep the issue of non-strictly convexity, we introduce a smoothed approximation of $F_{mix}(\vtau, \vtheta)$ that ``adds back" part of the missing $\HBtau$ term, 
\begin{equation*}
F_{mix}^{\epsilon}(\vtau, \vtheta) = 
\langle \vtheta , \vtau  \rangle + \HABtau + \epsilon \HBtau,
\end{equation*}
where $\epsilon$ is a small positive constant. 
\qiangnew{Similar smoothing techniques have also been applied to solve the standard MAP problem; see e.g., \citet{hazan2010norm, Meshi2012Conv}. 
We show in the following theorem that this smoothed dual approximation is closely connected to a direct approximation in the primal domain. }
\begin{thm}
\label{thm:smoothVersion}
Let $\epsilon$ be a positive constant, and $Q(\vx_B ; \vtheta)$ as defined in \eqref{equ:marginalMAP}. Define 
\begin{align*}
\Phi^{\epsilon}_{AB} (\vtheta) =  \log \big \{ [\sum_{\vx_B}  \exp(Q(\vx_B ; \vtheta))^{1/\epsilon}]^{\epsilon} \big \}, 
\end{align*}
then we have 
\begin{align}
\Phi^{\epsilon}_{AB} (\vtheta) = \max_{\vtau \in \M} \big \{\langle \vtheta , \vtau  \rangle + \HABtau + \epsilon \HBtau \big \}. 
\label{equ:smoothdual}
\end{align}
In addition, we have \[\lim_{\epsilon \to 0^+}\Phi^{\epsilon}_{AB}(\vtheta) = \Phi_{AB}(\vtheta),\] where $\epsilon\to 0^+$ denotes approaching zero from the positive side. 
\end{thm}
\begin{proof}
The proof is similar to that of Theorem~\ref{thm:duality}, but exploits the non-negativity of a weighted sum of two KL divergence terms, 
$$\mathrm{D}_{KL}[\tau(\vx_A | \vx_B) || p(x_{A} | \vx_B ; \vtheta)] + \epsilon \mathrm{D}_{KL}[\tau(\vx_B) || p(\vx_B)].$$
The remaining part follows directly from the standard zero temperature limit formula, 
\begin{align}
\label{equ:zerolimit}
\lim_{\epsilon \to 0^+} [\sum_{x} f(x)^{1/\epsilon}]^{\epsilon} = \max_{x} f(x), 
\end{align}
where $f(x)$ is any function with positive values.  
\end{proof}

\section{Variational Approximations for Marginal MAP}
\label{sec:variational}
Theorem~\ref{thm:duality} transforms the marginal MAP problem into a
variational form, but obviously does not decrease its computational hardness.  
Fortunately, many well-established variational techniques 
for sum- and max-inference can be extended to  
apply to \eqref{equ:mixduality}, opening a new door for deriving novel approximate
algorithms for marginal MAP. In the spirit of \citet{Wainwright08}, one can either relax $\margpoly$ to
a simpler outer bound like $\localpoly$ and replace $F_{mix}(\vtau, \vtheta)$ by
some tractable form to give algorithms similar to loopy BP or TRW BP, or
restrict $\margpoly$ to a tractable subset like $\margpoly_{mf}$ 
to give mean-field-like algorithms.  In the sequel,
we demonstrate several such approximation schemes, mainly focusing on the
BP-like methods with pairwise free energies. We will briefly discuss mean-field-like methods
when we connect to EM in section~\ref{sec:EM}, and derive an extension to junction graphs that exploits higher order approximations in Section~\ref{sec:junctiongraph}. 
Our framework can be easily adopted to take advantage of other, more advanced variational techniques,
like those using higher order cliques \citep[e.g.,][]{Yedidia_Bethe, globerson2007approximate,  liu11d, hazantightening} 
or more advanced optimization methods like dual decomposition \citep{Sontag_optbook} or alternating direction method of multipliers \citep{Boyd10}. 

We start by characterizing the graph structure on which marginal MAP is tractable. 
\begin{mydef}
\label{def:partialorder}
We call $G$ an \emph{$A$-$B$ tree} if there exists a partial order on the node set $V = A\cup B$, satisfying
\begin{description}
\item
{\bf 1) Tree-order}. For any $i \in V$, there is at most one other node $j \in V$ (called its parent), such that $j \prec i$ and $(ij) \in E$; 
\item
{\bf 2) A-B Consistency}. For any $a \in A$ and $b \in B$, we have $b \prec a$.
\end{description}
We call such a partial order an $A$-$B$ tree-order of $G$. 
\end{mydef}
For further notation, let $G_A = (A, E_A)$ be the subgraph induced by nodes in $A$,
i.e., $E_A = \{(ij)  \in E \colon i \in A, j\in A\}$,
and similarly for $G_B = (B, E_B)$. Let $\cross = \{ (ij)\in E \colon i \in A, j \in B \}$ be the
edges that join sets $A$ and $B$. 

Obviously, marginal MAP on an $A$-$B$ tree can be tractably solved by sequentially eliminating 
the variables along the $A$-$B$ tree-order \citep[see e.g.,][]{Koller_book}.
We show that its dual optimization is also tractable in this case. 
\begin{lem}
\label{lem:ABtree}
If $G$ is an $A$-$B$ tree, then
\begin{description}
\item[1)] The locally consistent polytope equals the marginal polytope, that is, $\margpoly = \localpoly$. 
\item[2)] The conditional entropy has a pairwise decomposition, 
\begin{align}
\label{equ:ABtree_entropy}
\HABtau = \sum_{i\in A} \Hitau \ \ -  \!\!\!\!\!\! \sum_{(ij)\in E_A\cup \cross} \!\!\!\!\!\!  \Iijtau . 
\end{align}
\end{description}
\end{lem}
\begin{proof}
1)\ The fact that $\margpoly = \localpoly$ on trees is a standard result; see \citet{Wainwright08} for details. \\
2)\ Because $G$ is an $A$-$B$ tree, both $p(\vx)$ and $p(\vx_B)$ have tree structured conditional dependency. We then have \citep[see e.g.,][]{Wainwright08} that
\begin{align*}
\Hxtau = \sum_{i\in V} \Hitau - \sum_{(ij)\in E} \Iijtau, &&\text{and}&& \HBtau = \sum_{i \in B} \Hitau - \sum_{(ij) \in E_B} \Iijtau.
\end{align*}
Equation \eqref{equ:ABtree_entropy} follows by using the entropic chain rule $\HABtau = \Hxtau - \HBtau$.
\end{proof}

\subsection{Bethe-like Free Energy} 

Lemma~\ref{lem:ABtree} suggests that the free energy of $A$-$B$ trees can be decomposed into singleton and pairwise terms that are easy to deal with. This is not true for general graphs, but motivates a ``Bethe" like approximation, 
%
 \begin{align}
& \Phi_{bethe}(\vtheta)= \max_{\vtau \in  \localpoly}   F_{bethe}(\vtau, \vtheta),  
& F_{bethe}(\vtau, \vtheta) = \langle \vtheta, \vtau \rangle \ + \   \sum_{i \in A} \Hitau  \ \!\!\ -  \!\!\!\!\!\! \!\!\! \sum_{(ij)\in E_A\cup \cross} \!\!\!\!\!\! \!\!\!  \Iijtau,  \label{equ:betheenergy}
 \end{align}
%
where $F_{bethe}(\vtau , \vtheta)$ is a  ``truncated" Bethe free energy, whose entropy and mutual 
information terms that involve only max nodes are truncated. 
If $G$ is an $A$-$B$ tree, $\Phi_{bethe}$ equals 
the true $\Phi_{AB}$, giving an intuitive justification.  In the sequel we 
give more general theoretical conditions under which this approximation gives the exact solution, and we find empirically that it usually gives surprisingly good solutions in practice.  
Similar to the regular Bethe approximation, \eqref{equ:betheenergy} leads to a nonconvex 
optimization, and we will derive both message passing algorithms and provably convergent algorithms to solve it. 

%

\subsection{Tree-reweighted Free Energy}
\label{sec:trw_bp_marginalMAP}
Following the idea of TRW belief propagation \citep{Wainwright_TRBP}, we construct an approximation of marginal MAP 
using a convex combination of $A$-$B$ subtrees  (subgraphs of $G$ that are $A$-$B$ trees). 
Let $\mathcal{T}_{AB}$ be a collection of $A$-$B$ subtrees of $G$. We assign with each $T \in \mathcal{T}_{AB}$ a weight $w_T$ satisfying $w_T\geq 0$ and $\sum_{T\in \mathcal{T}_{AB}}{w_T} = 1$.  
For each $A$-$B$ sub-tree $T = (V, E_T)$, define 
\begin{equation*}
\HABtauT = \sum_{i\in A} \Hitau  \, - \!\!\! \sum_{(ij) \in E_T\backslash E_B} \!\!\! \Iijtau . 
\end{equation*}
As shown in \citet{Wainwright08}, the $\HABtauT$ is always a concave function of $\vtau$ on $\localpoly$, and $\HABtau \leq \HABtauT$ for all $\vtau \in \margpoly$ and $T\in \mathcal{T}_{AB}$. More generally, we have
$\HAB \leq \sum_{T\in \mathcal{T}_{AB}}  w_T \HABtauT$,  which can be transformed to
\begin{equation}
\label{equ:trwBound}
\HABtau \leq  \sum_{i \in A} \Hitau \ \ -  \!\!\!\!\!\! \sum_{(ij)\in E_A\cup \cross}  \!\!\!\!\!\!  \rho_{ij}\Iijtau, 
\end{equation}
where $\rho_{ij} = \sum_{T: (ij)\in E_T} w_T$ 
are the edge appearance probabilities as defined in \citet{Wainwright08}. 
Replacing $\margpoly$ with $\localpoly$ and $\HABtau$ with the bound in \eqref{equ:trwBound} leads to a TRW-like approximation of marginal MAP, 
\begin{align}
&\Phi_{trw}(\vtheta) = \max_{\vtau \in \localpoly} F_{trw}(\vtau, \vtheta), 
& F_{trw}(\vtau, \vtheta) = \langle \vtheta, \vtau \rangle \  +  \   \sum_{i \in A} \Hitau \ \ -  \!\!\!\!\!\! \sum_{(ij)\in E_A\cup \cross}  \!\!\!\!\!\!  \rho_{ij}\Iijtau. 
\label{equ:trwmaxF}
\end{align}
Since $\localpoly$ is an outer bound of $\margpoly$, and $F_{trw}$ is a concave upper bound of the true free 
energy, we can guarantee that $\Phi_{trw}(\vtheta)$ is always an upper bound of $\Phi_{AB}(\vtheta)$. To our knowledge, this provides the first known convex relaxation for upper bounding marginal MAP. 
One can also optimize the weights $\{w_T \colon T\in \mathcal{T}_{AB}\}$ to get the tightest upper bound using methods similar to those used for regular TRW BP \citep[see][]{Wainwright_TRBP}.
\subsection{Global Optimality Guarantees} 
\label{sec:globalopt}
We show the global optimality guarantees of the above approximations under some circumstances. 
In this section, we always assume $G_A$ is a tree, and hence the objective function is tractable to calculate for a given
$\vx_B$. 
However, the optimization component remains intractable in this case, because the marginalization step destroys the decomposition structure of the objective function (see \figref{fig:hiddenchain}). 
It is thus nontrivial to see how the Bethe and TRW approximations behave in this case. 

In general, suppose we approximate $\Phi_{AB}(\vtheta)$ using the following pairwise approximation,  
\begin{equation}
\Phi_{tree}(\vtheta) = \max_{\vtau \in \localpoly}  \big \{ \langle \vtheta, \vtau \rangle  \  +  \ \sum_{i\in A} \Hitau - \!\!\!\sum_{(ij)\in E_A} \!\!\! \Iijtau - \!\!\!\sum_{(ij)\in \cross} \!\!\!\! \rho_{ij} \Iijtau \big \} , 
\label{equ:Phitree}
\end{equation}
where the weights on the sum part, $\{ \rho_{ij} \colon (ij) \in E_A\}$, have been fixed to be ones.  This choice makes sure that the sum part is ``intact'' in the
approximation, while the weights on the crossing edges, $\vrho_{AB} = \{\rho_{ij} \colon (ij)\in \cross \}$, can take arbitrary values, corresponding to different free energy approximation methods.  
If ${\rho}_{ij} = 1$ for $\forall (ij)\in \cross$, it is the Bethe free energy; it will correspond to the TRW free energy if $\{\rho_{ij}\}$ are taken to be a set of edge appearance probabilities (which in general have values less than one).  The edge appearance probabilities of $A$-$B$ trees are more restrictive than for the standard trees used in TRW BP.  
For example, if the max part of a $A$-$B$ sub-tree is a connected tree, then it can include at most one crossing edge, so in this case $\vrho_{AB}$ should satisfy $\sum_{(ij) \in \cross}\rho_{ij}  = 1$, $\rho_{ij} \geq 0$.
Interestingly, we will show in Section~\ref{sec:EM} that if $\rho_{ij} \rightarrow +\infty$ for $\forall (ij) \in \cross$, 
then Equation~\eqref{equ:Phitree} is closely related to an EM algorithm. 

\begin{thm} 
\label{thm:betheglobalopt}
Suppose the sum part $G_A$ is a tree, and we approximate $\Phi_{AB}(\vtheta)$ 
using $\Phi_{tree} (\vtheta)$ defined in \eqref{equ:Phitree}. Assume that \eqref{equ:Phitree} is \emph{globally} optimized. 
\begin{enumerate}
\item[{\rm (i)}] We have $\Phi_{tree} (\vtheta) \geq \Phi_{AB}(\vtheta)$. If there exists 
$\vx_B^*$ such that $Q(\vx_B^*; \vtheta) = \Phi_{tree}(\vtheta)$, we have 
$\Phi_{tree} (\vtheta)= \Phi_{AB}(\vtheta)$, and $\vx_B^*$ is a globally optimal marginal MAP solution.  
\item[{\rm (ii)}] Suppose $\vtau^*$ is a \emph{global} maximum of 
\eqref{equ:Phitree}, and $\{\tau^*_i(x_i)  \colon i \in B \}$ have integral values, i.e., 
$\tau^*_i(x_i) =0~\text{or}~1$, then $\{ x_i^* = \arg\max_{x_i} \tau_i^*(x_i)  \colon i \in B \}$ 
is a globally optimal solution of  the marginal MAP problem \eqref{equ:marginalMAP}.
\end{enumerate}\vspace{-.5\baselineskip}
\end{thm}
\begin{proof}[Proof (sketch)] (See appendix for the complete proof.)
The fact that the sum part $G_A$ is a tree guarantees the marginalization is exact.  Showing \eqref{equ:Phitree} is a relaxation of the maximization problem and applying standard relaxation arguments completes the proof. 
\end{proof}
\textbf{Remark.} Theorem~\ref{thm:betheglobalopt} works for arbitrary values of $\vrho_{AB}$, and suggests a fundamental tradeoff of hardness as $\vrho_{AB}$ takes on  different values. On the one hand, the value of $\vrho_{AB}$ controls the concavity of the objective function in \eqref{equ:Phitree} and hence the difficulty of finding a global optimum; small enough $\vrho_{AB}$ (as in TRW) can ensure that \eqref{equ:Phitree} is a convex optimization, while larger $\vrho_{AB}$ (as in Bethe or EM) causes \eqref{equ:Phitree} to become non-convex, making it difficult to apply Thoerem~\ref{thm:betheglobalopt}. 
On the other hand, the value of $\vrho_{AB}$ also controls how likely the solution is to be integral -- larger $\rho_{ij}$ emphasizes the mutual
information terms, forcing the solution towards integral points. Thus the solution of the TRW free energy is less likely to be integral
than the Bethe free energy, causing a difficulty in applying
Theorem~\ref{thm:betheglobalopt} to TRW solutions as well. 
The TRW approximation ($\sum_{ij} \rho_{ij} = 1$) and EM ($\rho_{ij} \rightarrow +\infty$; see Section~\ref{sec:EM}) reflect two extrema of this tradeoff between concavity and integrality, respectively, while the Bethe approximation ($\rho_{ij}=1$) appears to represent a reasonable compromise that often gives excellent performance in practice. In Section~\ref{sec:localoptimality}, we give a different set of local optimality guarantees that are derived from a reparameterization perspective. 

\section{Message Passing Algorithms for Marginal MAP}
\label{sec:message}
We now derive message-passing-style algorithms to optimize the ``truncated" Bethe or TRW free energies in \eqref{equ:betheenergy} and \eqref{equ:trwmaxF}. Instead of optimizing the truncated free energies directly, we leverage the results of Theorem~\ref{thm:smoothVersion} and consider their ``annealed" versions, 
\begin{equation*}
\max_{\vtau \in \localpoly}  \big \{ \langle \vtheta , \vtau \rangle + \hatHABtau + \epsilon \hatHBtau \big \},
\end{equation*}
where $\epsilon$ is a positive annealing coefficient (or temperature), and the $\hatHABtau$ and $\hatHBtau$ are the generic pairwise approximations of $\HABtau$ and $\HBtau$, respectively.  That is, 
\begin{align}
\hatHABtau = \sum_{i \in A} \Hitau \ \ -  \!\!\!\!\!\! \sum_{(ij)\in E_A\cup \cross}  \!\!\!\!\!\!  \rho_{ij}\Iijtau, &&\text{and}&&  \hatHBtau = \sum_{i \in B} \Hitau \ \ -  \!\! \sum_{(ij)\in E_B}  \!\!\!\! \rho_{ij}\Iijtau, 
\label{equ:HABHB}
\end{align}
where different values of pairwise weights $\{ \rho_{ij} \}$ correspond to either the Bethe approximation or the TRW approximation. This yields a generic pairwise free energy optimization problem,
\begin{align}
\max_{\vtau \in \localpoly}  \big \{  \langle \vtheta, \vtau \rangle + \sum_{i \in V} w_i \Hitau - \sum_{(ij)\in E}  {w_{ij}} \Iijtau \big \},
\label{equ:generalF}
\end{align}
where the weights $\{w_i, w_{ij}\}$ are determined by the temperature $\epsilon$ and $\{\rho_{ij}\}$ via
 \begin{align}
w_i = \left\{   \begin{array}{l l}
   1 & \quad \text{$\forall i \in A$}\\
    \epsilon& \quad \text{$\forall i \in B$} , \\
  \end{array} \right. &&
  w_{ij} = \left\{   \begin{array}{l l}
    \rho_{ij} & \quad \text{$\forall (ij) \in E_A \cup \cross$}\\  
    \epsilon \rho_{ij} & \quad \text{$\forall (ij) \in E_B$} . \\  
  \end{array} \right.
\label{equ:defineWeights}
  \end{align}
The general framework in \eqref{equ:generalF} provides a unified treatment for approximating sum-inference, max-inference and mixed, marginal MAP
problems simply by taking different weights. Specifically, 
\begin{enumerate}
\item If $w_i = 1$ for all $ i \in V$, Eq.~\eqref{equ:generalF} corresponds to the sum-inference problem and the sum-product BP objectives and algorithms. 
\item If $w_i \to 0^+$ for all $ i \in V$ \qiangold{(and the corresponding $w_{ij} \to 0^+$)}, Eq.~\eqref{equ:generalF} corresponds to the max-inference problem and the max-product linear programming objective and algorithms. 
\item If $w_i = 1$ for $\forall i \in A$ and $w_i = 0$ for $\forall i \in B$  \qiangold{(and the corresponding $w_{ij} \to 0^+$)}, Eq.~\eqref{equ:generalF} corresponds to the marginal MAP problem; in the sequel, we derive ``mixed-product" BP algorithms. 
\end{enumerate}
Note the different roles of the singleton and pairwise weights: the singleton weights $\{w_i \colon i \in V\}$ define the type of inference problem, while the pairwise weights $\{w_{ij} \colon (ij) \in E\}$ determine the approximation method (e.g., Bethe vs.\ TRW). 

\begin{algorithm}[tb]
\caption{Annealed BP for Marginal MAP} 
\label{alg:annealedmsg}
\begin{algorithmic}
\STATE Define the pairwise weights $\{\rho_{ij} \colon (ij) \in E\}$, e.g., $\rho_{ij}=1$ for Bethe or valid appearance probabilities for TRW. 
Initialize the messages $ \{ m_{i \to j} \colon (ij) \in E \}$. 
\FOR{iteration $t$} 
\STATE 1. Update $\epsilon$ by $\epsilon = 1/t$, and correspondingly the weights $\{w_i, w_{ij}\}$ by \eqref{equ:defineWeights}. 
\STATE 2. Perform the message passing update in \eqref{equ:weightedmsg} for all edges $(ij) \in E$.  
\ENDFOR
\STATE Calculate the singleton beliefs $b_{i}(x_i)$ and decode the solution $\vx_B^*$, 
 \begin{align}
 x_i^* =  \argmax_{x_i} b_{i}(x_i), ~~ \forall i \in B, & \text{ where $b_{i}(x_i) \propto \psi_i(x_i) m_{\sim i}(x_i)$} .
 \end{align}
\end{algorithmic}
\end{algorithm} 
%
\qiangold{
We now derive a message passing algorithm for solving the generic problem \eqref{equ:generalF}, using a Lagrange multiplier method similar to \citet{Yedidia_Bethe} or \citet{Wainwright_TRBP}. 
\vspace{-.8\baselineskip}
\begin{pro}
\label{pro:generalweightedBP}
Assuming $w_i$ and $w_{ij}$ are strictly positive, the stationary points of \eqref{equ:generalF} satisfy the fixed point condition of the following message passing update, 
\begin{align}
\label{equ:weightedmsg}
& \text{Message Update:} 
&&&& m_{i\to j}(x_j)  \gets   \big[ \sum_{x_i}  (\psi_i(x_i) m_{\sim i}(x_i))^{\frac{1}{w_i}} \left (\frac{ \psi_{ij}(x_i,x_j)}{m_{j\to i}(x_i)} \right) ^{\frac{1}{w_{ij}}}  \big]^{w_{ij}} ,    \\ 
& \text{Marginal Decoding:}  &&&&  \notag\\
&&&& & 
\hspace{-.25\textwidth}\tau_{i}(x_i) \propto \big [ \psi_i(x_i) m_{\sim i}(x_i) \big ] ^{\frac{1}{w_i}}, ~~~~ \tau_{ij}(x_{i}, x_j) \propto  \tau_i(x_i) \tau_j(x_j)  \left [ \frac{ \psi_{ij}(x_i,x_j)}{ m_{i\to j}(x_j) m_{j\to i}(x_i)  } \right ] ^{\frac{1}{w_{ij}} }, 
\label{equ:weighted_marginals}
\end{align}
where $\displaystyle m_{\sim i} (x_i) \coloneq \prod_{k \in \neib{i}}  m_{k\to i}(x_i)$ is the product of messages sent into node $i$,  and $\neib{i}$ is the set of neighboring nodes of $i$.
\end{pro}
\begin{proof}[Proof (sketch)] (See appendix for the complete proof.)
Note that \eqref{equ:weighted_marginals} is simply the KKT condition of \eqref{equ:generalF}, with the log of the message $\log m_{i\to j}$ being the Lagrange multipliers. 
Plugging \eqref{equ:weighted_marginals} into the local consistency constraints of $\localpoly$ in \eqref{equ:localpolydefine} gives \eqref{equ:weightedmsg}.  
\end{proof}
}
The above message update is mostly similar to TRW-BP of \citet{Wainwright_TRBP}, except that it incorporates general singleton weights $w_i$. 
The marginal MAP problem can be solved by running \eqref{equ:weightedmsg} with $\{w_i, w_{ij}\}$ defined by \eqref{equ:defineWeights} and a scheme for choosing the temperature $\epsilon$, either directly set to be a small constant, or gradually decreased (or annealed) to zero through iterations, e.g., by $\epsilon = 1/t$ where $t$ is the iteration. 
Algorithm~\ref{alg:annealedmsg} describes the details for the annealing method.  

\block{
\begin{algorithm}[htb]
\caption{Weighted Message Passing Algorithm for Solving the Generic Problem \eqref{equ:generalF}} 
\label{alg:weightedmsg}
\begin{algorithmic}
\STATE Initialize the messages $m_{i \to j} (x_j)$ for all edges $(ij) \in E$ and $x_j \in \X_j$. 
\FOR{iteration $t$}
\FOR{edge $(ij) \in E$}
\STATE {\emph{Perform message update}:}  
\begin{align}
\label{equ:weightedmsg}
&m_{i\to j}(x_j)  \gets     \big[ \sum_{x_i}  (\psi_i(x_i) m_{\sim i}(x_i))^{1/w_i} ({ \psi_{ij}(x_i,x_j)}/{m_{j\to i}(x_i)})^{1/w_{ij}}  \big]^{w_{ij}} ,   \\
&\text{where $m_{\sim i} (x_i) = \prod_{k \in \neib{i}}  m_{k\to i}(x_i)$.} \notag
\end{align}
\vspace{-1.5\baselineskip}
\ENDFOR
\ENDFOR
\STATE Calculate the singleton and pairwise marginals, 
\begin{align}
\label{equ:weighted_marginals}
\tau_{i}(x_i) \propto (\psi_i(x_i) m_{\sim i}(x_i))^{1/w_i} , &&
\tau_{ij}(x_{ij}) \propto  \tau_i(x_i) \tau_j(x_j) (\frac{ \psi_{ij}(x_i,x_j)}{m_{i\to j}(x_j) m_{j\to i}(x_i)}) ^{1/w_{ij} } 
\end{align}
\end{algorithmic}
\end{algorithm} 
}
%
%

%

\subsection{Mixed-Product Belief Propagation}

Directly taking $\epsilon \to 0^+$ in message update \eqref{equ:weightedmsg}, we can get an interesting ``mixed-product" BP algorithm that is a hybrid of the max-product and sum-product message updates, with a novel ``argmax-product" message update that is specific to marginal MAP problems. 
This algorithm is listed in Algorithm~\ref{alg:mix_product_msg}, and described by the following proposition:
\begin{algorithm}[t]
\caption{Mixed-product Belief Propagation for Marginal MAP} 
\label{alg:mix_product_msg}
\begin{algorithmic}
\STATE Define the pairwise weights $\{\rho_{ij} \colon (ij) \in E\}$ and initialize messages $ \{ m_{i \to j} \colon (ij) \in E \}$ as in Algorithm~\ref{alg:annealedmsg}. 
\FOR{iteration $t$}
\FOR{edge $(ij) \in E$}
\STATE{\emph{Perform different message updates depending on the node type of the source and destination},}  
\begin{align}
&\hspace{-0.05\textwidth}\begin{tabular}{c}{$A \to A\cup B$:}\\ {\small (sum-product)}\end{tabular}  &  &m_{i\to j}(x_j) \gets     \big[ \sum_{x_i}  (\psi_i(x_i) m_{\sim i}(x_i)) (\frac{ \psi_{ij}(x_i,x_j)}{m_{j\to i}(x_i)})^{1/\rho_{ij}}  \big]^{\rho_{ij}},  \label{equ:mix_product1}\\
&\hspace{-0.05\textwidth}\begin{tabular}{c}{$B \to B$:}\\{\small (max-product)}\end{tabular}  & &m_{i\to j}(x_j)  \gets     \max_{x_i}  (\psi_i(x_i) m_{\sim i}(x_i))^{\rho_{ij}}(\frac{ \psi_{ij}(x_i,x_j) }{m_{j\to i}(x_i)}) ,    \label{equ:mix_product2}\\ 
&\hspace{-0.05\textwidth}\begin{tabular}{c}{$B \to A$:}\\{\small (argmax-product)}\end{tabular}  &  &m_{i\to j}(x_j) \gets     \big[ \sum_{x_i \in \X_i^*} (\psi_i(x_i) m_{\sim i}(x_i)) (\frac{ \psi_{ij}(x_i,x_j)}{m_{j\to i}(x_i)})^{1/\rho_{ij}}  \big]^{\rho_{ij}} ,    \label{equ:mix_product3}\\
&&&\hspace{-0.2\textwidth}\text{where the set $\X_i^* = \argmax_{x_i} \psi_i(x_i) m_{\sim i}(x_i)$ 
and $m_{\sim i}(x_i)= \prod_{k\in \neib{i}} m_{ki}(x_i)$.}   \notag 
\end{align}
\vspace{-1.5\baselineskip}
\ENDFOR
\ENDFOR
\STATE Calculate the singleton beliefs $b_{i}(x_i)$ and decode the solution $\vx_B^*$, 
 \begin{align}
 x_i^* =  \argmax_{x_i} b_{i}(x_i), ~~ \forall i \in B, & \text{ where $b_{i}(x_i) \propto \psi_i(x_i) m_{\sim i}(x_i)$} .
 \end{align}
\end{algorithmic}
\end{algorithm} 
\begin{pro}
As $\epsilon$ approaches zero from the positive side, that is, $\epsilon \to 0^+$, the message update \eqref{equ:weightedmsg} reduces to the update in \eqref{equ:mix_product1}-\eqref{equ:mix_product3} in Algorithm~\ref{alg:mix_product_msg}. 
\end{pro}
\begin{proof}
For messages from $i\in A$ to $j\in A\cup B$, we have $w_i = 1$, $w_{ij} = \rho_{ij}$; the result is obvious. \\
 For messages from $i\in B$ to $j\in B$, we have $w_i = \epsilon$, $w_{ij}= \epsilon \rho_{ij}$. The result follows from the zero temperature limit formula in \eqref{equ:zerolimit}, by letting 
  $f(x_i) =  (\psi_i(x_i) m_{\sim i}(x_i))^{\rho_{ij}} (\frac{\psi_{ij}(x_i,x_j)}{m_{j\to i}(x_i)})$.
  \\
For messages from $i\in B$ to $j\in A$, we have $w_i = \epsilon$, $w_{ij} = \rho_{ij}$. 
One can show that
$$\lim_{\epsilon \to 0^+}   \Big[ \frac{\psi_i(x_i) m_{\sim i}(x_i) }{   \max_{x_i}  \psi_i(x_i) m_{\sim i}(x_i)     } \Big]^{1/\epsilon}  =  \deltaIndic(x_i \in    \X_i^* ),$$
where $\X_i^* = \argmax_{x_i} \psi_i(x_i) m_{\sim i}(x_i)$. 
Plugging this into \eqref{equ:weightedmsg} and dropping the constant term, we get the message update in \eqref{equ:mix_product3}. 
\end{proof}

Algorithm~\ref{alg:mix_product_msg} has an intuitive interpretation: the sum-product and max-product messages in \eqref{equ:mix_product1} and \eqref{equ:mix_product2} correspond to the marginalization and maximization steps, respectively. The special ``argmax-product" messages in \eqref{equ:mix_product3} serves to synchronize the sum-product and max-product messages -- it restricts the max nodes to the currently decoded local marginal MAP solutions $\X^*_i = \argmax \psi_i(x_i) m_{\sim i}(x_i)$, and passes the posterior beliefs back to the sum part. Note that the summation notation in \eqref{equ:mix_product3} can be ignored if $\X^*_i$ has only a single optimal state.

One critical feature of our mixed-product BP is that it takes simultaneous movements on the marginalization and maximization sub-problems in a parallel fashion, 
and is computationally much more efficient than the traditional methods that require fully solving a marginalization sub-problem before taking each maximization step. This advantage is inherited from our general variational framework, which naturally integrates the marginalization and maximization sub-problems into a joint optimization problem. 

Interestingly, Algorithm~\ref{alg:mix_product_msg} also bears similarity to a recent hybrid message passing method of \citet{Jiang10}, which differs from Algorithm~\ref{alg:mix_product_msg} only in replacing the special argmax-product messages \eqref{equ:mix_product3}  with regular max-product messages.
We make a detailed comparison of these two algorithms in Section~\ref{sec:compare_jiang}, and show that it is in fact the argmax-product messages \eqref{equ:mix_product3} that lends our algorithm several appealing optimality guarantees.

\subsection{Reparameterization Interpretation and Local Optimality Guarantees}
\label{sec:localoptimality}
An important interpretation of the sum-product and max-product BP is the 
reparameterization viewpoint \citep{Wainwright03, Weiss07}: Message passing updates can be 
viewed as moving probability mass between local pseudo-marginals (or beliefs), in a way that leaves their
product a reparameterization of the original distribution, while ensuring some consistency conditions at the fixed points. Such viewpoints are theoretically important, because they are useful for proving optimality guarantees for the BP algorithms.  
In this section, we show that the mixed-product BP in Algorithm~\ref{alg:mix_product_msg} has a similar reparameterization interpretation, based on which we establish a local optimality guarantee for mixed-product BP. 


To start, we define a set of ``mixed-beliefs" as 
\begin{align}
\margin_{i}(x_i)  \propto \psi_i(x_i) m_{\sim i}(x_i), &&
\margin_{ij}(x_{ij}) \propto  \margin_i(x_i) \margin_j(x_j) \left [ \frac{ \psi_{ij}(x_i, x_j)}{m_{i\to j}(x_j) m_{j\to i}(x_i)}\right ] ^{1/\rho_{ij} }. 
\label{equ:mixedmargin}
\end{align}
The marginal MAP solution should be decoded from  $x_i^* \in \arg\max_{x_i} \margin_i(x_i),
\forall i \in B$, as is typical in max-product BP. 
\newcommand{\epmarign}{\margin()}
Note that the above mixed-beliefs $\{b_i, b_{ij}\}$ are different from the local marginals $\{\tau_i, \tau_{ij}\}$ defined in \eqref{equ:weighted_marginals}, but are rather softened versions of $\{\tau_i, \tau_{ij}\}$.Their relationship is explicitly clarified in the following.  
\begin{pro}
The $\{\tau_i, \tau_{ij}\}$ in \eqref{equ:weighted_marginals} and the $\{\margin_i, \margin_{ij}\}$ in \eqref{equ:mixedmargin} are associated via, 
\begin{align*}
\begin{cases}
\margin_{i} \propto \tau_{i} &\forall i\in A , \\
\margin_{i}\propto (\tau_{i})^{\epsilon} &\forall i\in B  
\end{cases} & &
\begin{cases}
\margin_{ij} \propto \margin_{i} \margin_{j}  (\frac{\tau_{ij} }{\tau_{i}  \tau_{j}}) & \forall (ij)\in E_A \cup \cross \\
\margin_{ij} \propto \margin_{i} \margin_{j}  (\frac{\tau_{ij} }{\tau_{i}  \tau_{j}})^{\epsilon} & \forall (ij)\in E_B .
\end{cases}
\end{align*}
\end{pro}
\begin{proof}
 Result follows from the simple algebraic transformation between \eqref{equ:weighted_marginals} and \eqref{equ:mixedmargin}. 
\end{proof}
Therefore, as $\epsilon \to 0^+$, the $\tau_i$ ($= b_i^{1/\epsilon}$) for $i\in B$ should concentrate their mass on a deterministic configuration, but $b_i$ may continue to have soft values. 

We now show that the mixed-beliefs $\{b_i, b_{ij} \}$ have a reparameterization interpretation. 
\begin{thm}
\label{thm:reparameter}
At the fixed point of mixed-product BP in Algorithm~\ref{alg:mix_product_msg} , the mixed-beliefs defined in \eqref{equ:mixedmargin} satisfy\\
\textbf{Reparameterization:}
\begin{equation}
\label{equ:repara}
p(\vx) \propto \prod_{i\in V} \margin_{i} (x_i) \prod_{(ij)\in E}\big[ \frac{\margin_{ij}(x_i, x_j)}{\margin_{i}(x_i) \margin_{j}(x_j)} \big]^{\rho_{ij}}  .
\end{equation}
\textbf{Mixed-consistency:}
\begin{align}
{\rm (a)}\hspace{-1em}&& \sum_{x_i} b_{ij}(x_i, x_j) & = b_j(x_j), & \forall  i \in A, j \in A\cup B  ,  \label{equ:sum_consistency}\\
{\rm (b)}\hspace{-1em}&&\max_{x_i} b_{ij}(x_i, x_j)  &= b_j(x_j), &\forall  i \in B, j \in B  , \label{equ:max_consistency} \\
{\rm (c)}\hspace{-1em}&&\sum_{x_i \in \arg\max b_i}    \!\!\!\!\!\!\!\!  b_{ij}(x_i, x_j) & = b_j(x_j), & \forall  i \in B, j \in A  . \label{equ:mix_consistency}
\end{align}
\end{thm}
\begin{proof}
Directly substitute the definition \eqref{equ:mixedmargin} into the message update \eqref{equ:mix_product1}-\eqref{equ:mix_product3}. 
\end{proof}

The three mixed-consistency constraints exactly map to the three
types of message updates in Algorithm~\ref{alg:mix_product_msg}. 
Constraint (a) and (b) enforces the regular sum- and max- consistency of the sum- and max- product messages in \eqref{equ:mix_product1} and \eqref{equ:mix_product2}, respectively. 
 Constraint (c) corresponds to the argmax-product message update in \eqref{equ:mix_product3}: it enforces the marginals to be consistent after $x_i$ is assigned to the currently decoded  solution, 
 $x_i =\arg \max_{x_i} b_i (x_i) =  \arg \max_{x_i} \sum_{x_j} b_{ij}(x_i, x_j)$, corresponding to solving a local marginal MAP problem on $b_{ij}(x_i, x_j)$. 
It turns out that this special constraint is a crucial ingredient of mixed-product BP, enabling us to prove guarantees on the strong local optimality of the solution. 

Some notation is required. Suppose $C$ is a subset of {max} nodes in $B$. Let $G_{C\cup A} = (C\cup A, E_{C\cup A})$ be
the subgraph of $G$ induced by nodes $C\cup A$, where $E_{C\cup A} = \{(ij) \in E \colon i,
j \in C\cup A\}$. We call $G_{C\cup A}$ a semi-$A$-$B$ subtree of $G$ if the edges in $E_{C\cup A}
\backslash E_B$ form an $A$-$B$ tree. In other words, $G_{C\cup A}$ is a semi-$A$-$B$
tree if it is an $A$-$B$ tree when ignoring any edges entirely within the {max} set $B$. 
See \figref{fig:semiABtree} for examples of semi $A$-$B$ trees. 

Following \citet{Weiss07}, we say that a set of weights $\{\rho_{ij} \}$ is \emph{provably convex} if there exist positive constants $\kappa_i $ and $\kappa_{i\to j}$, such that $\kappa_i  + \sum_{i' \in \neib{i}}\kappa_{i' \to i} =
1$ and $\kappa_{i\to j} + \kappa_{j\to i} = \rho_{ij}$. 
\citet{Weiss07} shows that if $\{\rho_{ij} \}$ is provably convex, then $H(\vtau) = \sum_i \Hitau - \sum_{ij} \rho_{ij} \Iijtau$ is a concave function of $\vtau$ in the locally consistent polytope $\localpoly$. 
\begin{thm}
\label{thm:localopt}
Suppose $C$ is a subset of $B$ such that $G_{C\cup A}$ is a semi-$A$-$B$ tree, and the weights $\{
\rho_{ij} \}$ satisfy 
\begin{enumerate}
\item $\rho_{ij} =1$ for $(ij)\in E_A$;
\item $0\leq \rho_{ij} \leq 1$ for $(ij)\in E_{C\cup A} \cap \cross$;
\item $\{\rho_{ij} \colon (ij)\in E_{C\cup A} \cap E_B\}$ is provably convex.
\end{enumerate}\vspace{-.5\baselineskip}
At the fixed point of mixed-product BP in Algorithm~\ref{alg:mix_product_msg}, if the mixed-beliefs on the max nodes $\{b_i, b_{ij} \colon i, j \in B\}$ defined in  \eqref{equ:mixedmargin} all have unique maxima, then there exists a $B$-configuration $\vx_B^*$ satisfying $x^*_i = \arg \max b_i$ for $\forall i
\in B$ and $(x^*_i, x^*_j) = \arg\max b_{ij}$ for $\forall (ij)\in E_B$, and $\vx_B^*$ is locally optimal in the sense that $Q(\vx_B^*; \vtheta)$ is not
smaller than any $B$-configuration that differs from $\vx_B^*$ only on $C$, that is, $Q(\vx_B^* ; \vtheta) = \max_{\vx_C}Q([\vx_C, x_{B\setminus C}^*] ; \vtheta)$. 
\end{thm} 
\begin{proof}[Proof (sketch)]
(See appendix for the complete proof.) 
The mixed-consistency constraint (c) in \eqref{equ:mix_consistency} and the fact that $G_{C\cup A}$ is a semi-$A$-$B$ tree enables the summation part 
to be eliminated away. The remaining part only involves the {max} nodes, and 
the method in \citet{Weiss07} for analyzing standard MAP can be applied.  
\end{proof}
%
%
\qiangnew{\textbf{Remark.} The proof of Theorem~\ref{thm:localopt} relies on transforming the marginal MAP problem to a standard MAP problem by eliminating the summation part. Therefore, variants of Theorem~\ref{thm:localopt} may be derived using other global optimality conditions of convexified belief propagation or linear programming algorithms for MAP, such as those in \citet{werner2007linear, werner2010revisiting, wainwright2005map}. We leave this to future work.}

For $G_{C\cup A}$ to be a semi $A$-$B$ tree, the sum part $G_A$ must be a tree, which 
Theorem~\ref{thm:localopt} assumes implicitly. For the hidden Markov chain in \figref{fig:hiddenchain},
Theorem~\ref{thm:localopt} implies only the local 
optimality up to Hamming distance one (or coordinate-wise optimality), because any semi $A$-$B$ subtree of $G$ in \figref{fig:hiddenchain} can contain at most one max node. 
However, Theorem~\ref{thm:localopt} is in general much stronger, especially when the {sum} part is not fully connected, or when the {max} part has interior regions disconnected from the {sum} part. As examples, see \figref{fig:semiABtree}(b)-(c).
\begin{figure}[tbp]
   \centering
   \begin{tabular}{ccccc}
   \includegraphics[height=0.1\textwidth]{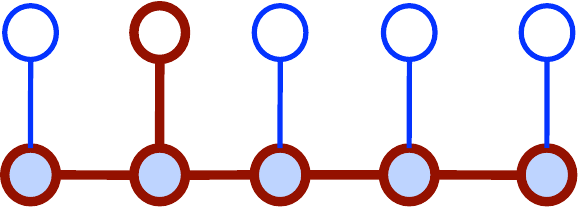}  & &   
   \includegraphics[height=0.1\textwidth]{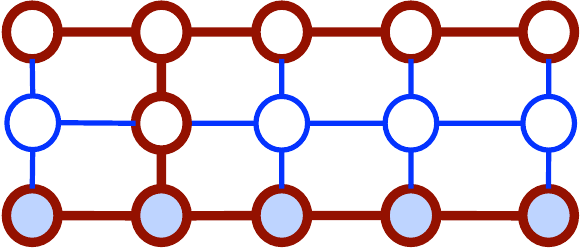}  && 
   \includegraphics[height=0.1\textwidth]{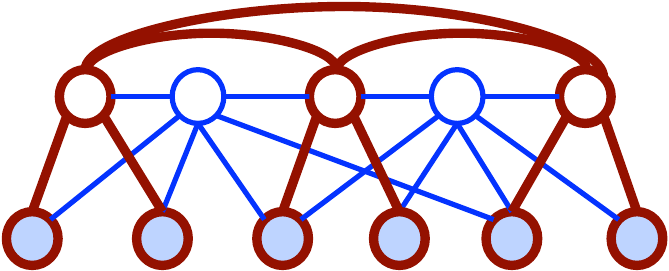}    \\
   (a) && (b) && (c)
   \end{tabular}
   \caption{Examples of semi $A$-$B$ trees. The shaded nodes represent sum nodes, while the unshaded are max nodes. 
   In each graph, a semi $A$-$B$ tree is labeled by red bold lines.
   Under the conditions of Theorem~\ref{thm:localopt},  the fixed point of mixed-product BP is locally optimal up to jointly perturbing all the max nodes in any semi-A-B subtree of $G$. }
   \label{fig:semiABtree}
\end{figure}

\subsection{The importance of the Argmax-product Message Updates}
\label{sec:compare_jiang}
\citet{Jiang10} proposed a similar hybrid message passing algorithm, repeated here as Algorithm~\ref{alg:jiang_product}, which differs from our mixed-product BP only in replacing our argmax-product message update \eqref{equ:mix_product3} with the usual max-product message update \eqref{equ:mix_product2}. \qiangnew{We show in this section that this very difference gives Algorithm~\ref{alg:jiang_product} very different properties, and fewer optimality guarantees, than our mixed-product BP. }
\begin{algorithm}[h]
\caption{ Hybrid Message Passing by \citet{Jiang10} } 
\label{alg:jiang_product}
\begin{algorithmic}
\STATE 1. Message Update: 
\begin{align*}
&\begin{tabular}{c}{$A \to A \cup B$}:\\ {(\small sum-product)}\end{tabular} 
&& m_{i\to j}(x_j) \gets     \big[ \sum_{x_i}  (\psi_i(x_i) m_{\sim i}(x_i)) (\frac{ \psi_{ij}(x_i,x_j)}{m_{j\to i}(x_i)})^{1/\rho_{ij}}  \big]^{\rho_{ij}} ,\\
&
\begin{tabular}{c}{$A \to A \cup B$}:\\ {(\small max-product)}\end{tabular} 
&& m_{i\to j}(x_j)  \gets     \max_{x_i}  (\psi_i(x_i) m_{\sim i}(x_i))^{\rho_{ij}} (\frac{ \psi_{ij}(x_i,x_j)}{m_{j\to i}(x_i)}) . 
\end{align*}
\STATE 2. Decoding:\quad $x_i^{*} = \argmax_{x_i} b_i(x_i)$ for $\forall i\in B$, where $b_i(x_i) \propto \psi_i(x_i) m_{\sim i}(x_i)$.
\end{algorithmic}
\end{algorithm}

Similar to our mixed-product BP, Algorithm~\ref{alg:jiang_product} also satisfies the reparameterization property in \eqref{equ:repara} (with beliefs $\{b_i, b_{ij}\}$ defined by \eqref{equ:mixedmargin});
 it also satisfies a set of similar, but crucially different, consistency conditions at its fixed points,
\begin{align*}
\sum_{x_i} b_{ij}(x_i, x_j) = b_j(x_j), ~~~~~~~~~~~  \forall i\in A, j \in A\cup B, \\
\max_{x_i} b_{ij}(x_i, x_j) = b_j (x_j), ~~~~~~~~~~~  \forall i\in B, j \in A\cup B, 
\end{align*}
which exactly map to the max- and sum- product message updates in Algorithm~\ref{alg:jiang_product}. 

Despite its striking similarity, Algorithm~\ref{alg:jiang_product} has very different properties, and does not share the appealing variational interpretation and optimality guarantees that we have demonstrated for mixed-product BP. 
First, it is unclear whether Algorithm~\ref{alg:jiang_product} can be interpreted as a fixed point algorithm for maximizing our, or a similar, variational objective function. 
Second, it does not inherit the same optimality guarantees in Theorem~\ref{thm:localopt}, despite its similar reparameterization and consistency conditions. 
These disadvantages are caused by the \qiangold{miss} 
 of the special argmax-product message update and its associated mixed-consistency condition in \eqref{equ:mix_consistency}, which was a critical ingredient of the proof of Theorem~\ref{thm:localopt}.

More detailed insights into Algorithm~\ref{alg:jiang_product} and mixed-product BP can be obtained by considering the special case when the full graph $G$ is an undirected tree. We show that in this case, Algorithm~\ref{alg:jiang_product} can be viewed as optimizing a set of \emph{approximate} objective functions, obtained by rearranging the max and sum operators into orders that require less computational cost, while 
mixed-product BP attempts to maximize the \emph{exact} objective function by message updates that effectively perform some ``asynchronous" coordinate descent steps. In the sequel, we use an illustrative toy example to explain the main ideas. 

\vspace{1\baselineskip}

\begin{wrapfigure}{r}{0.18\textwidth}
  \begin{center}
  \vspace{-25pt}
    \includegraphics[width=0.18\textwidth]{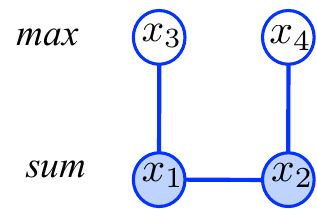}
      \vspace{-25pt}
  \end{center}
\end{wrapfigure}
\textbf{Example 2.} \textit{Consider the marginal MAP problem shown on the right, where the graph $G$ is an undirected tree;  the sum and max sets are  $A=\{1,2\}$ and $B=\{3,4\}$, respectively. We analyze how Algorithm~\ref{alg:jiang_product} and mixed-product BP in Algorithm~\ref{alg:mix_product_msg} perform on this toy example, when both taking Bethe weights ($\rho_{ij} = 1$ for $(ij)\in E$). }

\textit{
\emph{\textbf{Algorithm~\ref{alg:jiang_product} (\citet{Jiang10})}}. Since $G$ is a tree, one can show that Algorithm 3 (with Bethe weights) terminates after a full forward and backward iteration (e.g., messages passed along $x_3\to x_1 \to x_2 \to x_4$ and then $x_4 \to x_2 \to x_1 \to x_3$). By tracking the messages, one can write its final decoded solution in a closed form, 
\begin{align*}
x_3^* = \argmax_{x_3}\sum_{x_1}\sum_{x_2} \max_{x_4}[ \exp(\theta(\vx))], && x_4^* = \argmax_{x_4}\sum_{x_2}\sum_{x_1} \max_{x_3} [\exp(\theta(\vx))],  
\end{align*}
On the other hand, the true marginal MAP solution is given by, 
\begin{align*}
x_3^* = \argmax_{x_3} \max_{x_4} \sum_{x_1}\sum_{x_2} [ \exp(\theta(\vx))], && x_4^* = \argmax_{x_4} \max_{x_3}\sum_{x_2}\sum_{x_1} [\exp(\theta(\vx))]. 
\end{align*}
Here, Algorithm~\ref{alg:jiang_product} approximates the exact marginal MAP problem by rearranging the max and sum operators into an elimination order that makes the calculation easier. A similar property holds for the general case when $G$ is undirected tree: Algorithm 3 (with Bethe weights) terminates in a finite number of steps, and its output solution $x_i^*$ effectively maximizes an approximate objective function obtained by reordering the max and sum operators along a tree-order (see Definition~\ref{def:partialorder}) that is rooted at node $i$. 
The performance of the algorithm should be related to the error caused by exchanging the order of max and sum operators. However, exact optimality guarantees are likely difficult to show because it maximizes an inexact objective function.  
In addition, since each component $x_i^*$ uses a different order of arrangement, and hence maximizes a different surrogate objective function, it is unclear whether the joint $B$-configuration $\vx_B^* = \{x_i^* \colon i \in B\}$ given by Algorithm~\ref{alg:jiang_product} maximizes a single consistent objective function. 
}

\textit{
\emph{\textbf{Algorithm~\ref{alg:mix_product_msg}  (mixed-product)}}. On the other hand, the mixed-product belief propagation in Algorithm~\ref{alg:mix_product_msg} may not terminate in a finite number of steps, nor does it necessarily yield a closed form solution when $G$ is an undirected tree. 
However, Algorithm~\ref{alg:mix_product_msg} proceeds in an attempt to optimize the exact objective function. 
In this toy example, we can show that the true solution is guaranteed to be a fixed point of Algorithm~\ref{alg:mix_product_msg}. 
Let $b_3(x_3)$ be the mixed-belief on $x_3$ at the current iteration, and $x_3^* = \argmax_{x_3} b_3(x_3)$ its unique maxima. 
After a message sequence passed from $x_3$ to $x_4$, one can show that $b_4(x_4)$ and $x_4^*$ update to
\begin{align*}
x_4^* &= \argmax_{x_4}  b_4(x_4),  &&
b_4(x_4) =  \sum_{x_2} \sum_{x_1} \exp(\theta([x_3^*, x_{\neg 3} ])) = \exp( Q([x_3^*, x_{4}] ; \vtheta)), 
\end{align*}
where we maximize the exact objective function $Q([x_3, x_4] ; \vtheta)$ with fixed $x_3 = x_3^*$. 
Therefore, on this toy example, one sweep ($x_3\to x_4$ or $x_4 \to x_3$) of Algorithm~\ref{alg:mix_product_msg} is effectively performing a coordinate descent step, which monotonically improves the true objective function towards a local maximum.  
In more general models, Algorithm~\ref{alg:mix_product_msg} differs from sequential coordinate descent, and does not guarantee monotonic convergence. But, it can be viewed as a ``parallel" version of coordinate descent, which ensures the stronger local optimality guarantees shown in Theorem~\ref{thm:localopt}. 
}


\section{Convergent Algorithms by Proximal Point Methods }
\label{sec:proximal}
An obvious disadvantage of mixed-product BP is its lack of convergence guarantees, even when $G$ is an undirected tree. In this section, we apply a proximal point approach \citep[e.g.,][]{martinet1970breve, Rockafellar76} to derive convergent algorithms that directly optimize our free energy objectives, \qiangnew{which take the form of transforming marginal MAP into a sequence of pure (or annealed) sum-inference tasks.}  
Similar methods have been applied to standard sum-inference \citep{Yuille_CCCP} and max-inference \citep{Ravikumar10}. 

For the purpose of illustration,  we first consider the problem of maximizing the \emph{exact} marginal MAP free energy, $F_{mix}(\vtau, \vtheta) = \langle \vtau, \vtheta \rangle + H_{A|B}(\vtau)$. The proximal point algorithm works by iteratively optimizing a smoothed problem,  
$$\vtau^{t+1} =  \argmin_{\vtau \in \margpoly}  \{  - F_{mix}(\vtau, \vtheta)  +  \epsilontt D( \vtau ||  \vtau^{t})  \}, $$
where $\vtau^{t}$ is the solution at iteration $t$, and $\epsilontt$ is a positive
coefficient.  Here, $D(\cdot || \cdot)$ is a distance, called the proximal function, which forces $\vtau^{t+1}$ to be close to $\vtau^{t}$; 
typical choices of $D(\cdot || \cdot)$ are Euclidean or Bregman distances or $\psi$-divergences
\citep[e.g.,][]{Teboulle92, Iusem93}. 
Proximal algorithms have nice convergence guarantees:  
the objective series $\{f(\vtau^t) \}$ is guaranteed to 
be non-increasing at each iteration, and $\{\vtau^{t}\}$ converges to an
optimal solution, 
under some regularity conditions. See, e.g.,
\citet{Rockafellar76, Tseng93, Iusem93}. The proximal algorithm is closely related to the majorize-minimize (MM) algorithm \citep{Hunter04} and the convex-concave procedure \citep{Yuille_CCCP}.




For our purpose, we take $D(\cdot || \cdot)$ to be a KL divergence between distributions on the max nodes, 
$$D(\vtau || \vtau^t)  =  \mathrm{KL}(\tau_B(\vx_B) || \tau^t_B(\vx_B) ) = \sum_{\vx_B} \tau_B(\vx_B ) \log  \frac{\tau_B(\vx_B)}{\tau^t_B(\vx_B)}.$$
In this case, the proximal point algorithm reduces to Algorithm~\ref{alg:proximal_point}, which iteratively solves a smoothed free energy objective, with natural parameter $\vtheta^t$ updated at each iteration. 
\begin{algorithm}[t]
\caption{Proximal Point Algorithm for Marginal MAP (Exact)} 
\label{alg:proximal_point}
\begin{algorithmic}
\STATE Initialize local marginals $\vtau^0$. 
\FOR{iteration $t$}
\STATE 
\vspace{-1.2\baselineskip}
\begin{align}
& \vtheta^{t+1} = \vtheta + \epsilontt \log  \vtau^{t}_B, \label{equ:proximal_inner0}  \\
&\vtau^{t+1}  = \arg\max_{\tau \in \margpoly} \{   \langle \vtau, \vtheta^{t+1} \rangle + H_{A|B}(\vtau) + \epsilontt H_B(\vtau) \},  \label{equ:proximal_inner}
\end{align}
\vspace{-1.2\baselineskip}
\ENDFOR
\STATE Decoding: 
$\displaystyle x_i^{*} = \argmax_{x_i} \tau_i(x_i)$ for $\forall i\in B$. 
\end{algorithmic}
\end{algorithm} 
Intuitively, the proximal inner loop \eqref{equ:proximal_inner0}-\eqref{equ:proximal_inner} essentially ``adds back'' the truncated entropy
term $H_B(\vtau)$, while canceling its effect by adjusting $\vtheta$ in the opposite
direction.  
Typical choices of $\epsilontt$ include $\epsilontt = 1$ (constant) and $\epsilontt = 1/t$ (harmonic). 
Note that the proximal approach is distinct from an annealing method, which would require that the annealing coefficient vanish to zero. 
Interestingly, if we take $\epsilontt = 1$, then the inner maximization problem \eqref{equ:proximal_inner} reduces to the standard log-partition function duality \eqref{equ:sumduality}, corresponding to a pure marginalization task. This has the interpretation of transforming the marginal MAP problem into a sequence of standard sum-inference problems. 

In practice we approximate $\HAB$ and $\HB$ by pairwise entropy decomposition $\hatHABtau$ and $\hatHBtau$ in \eqref{equ:HABHB}, respectively.  
If $\hatHBtau$ is provably convex in the sense of \citet{Weiss07}, that is, there exist positive constants $\{\kappa_i, \kappa_{i\to j}\}$ satisfying $\rho_i = \kappa_i + \sum_{k\in \neib{i}} \kappa_{k\to i}$ and $\rho_{ij} = \kappa_{i\to j} + \kappa_{j \to i}$ for $i,j\in B$.
Then the resulting approximate algorithm can be interpreted as a proximal algorithm that 
maximizes $\hat{F}_{mix}(\vtau, \vtheta)$ with proximal function as 
\begin{equation*}
D_{pair}(\vtau || \vtau^t)  = \sum_{i \in B} \kappa_i  \mathrm{KL}[\tau_i (x_i)|| \tau_i^0(x_i)] ~ + \!  \sum_{(ij) \in E_B} \kappa_{i\to j}  \mathrm{KL}[(\tau_{ij}(x_i | x_j) || \tau_{ij}^0(x_i | x_j)  ]. 
\end{equation*}
In this case, Algorithm~\ref{alg:proximal_point} is still a valid proximal algorithm and inherits its convergence guarantees. 
In practice one uses approximations that are not provably convex. An interesting special case is when both $\HAB$ and $\HB$ are approximated by a Bethe approximation.
\qiangnew{This has the effect that the optimization \eqref{equ:proximal_inner} can be solved using standard belief propagation.
Although the Bethe form for $\HAB$ and $\HB$ 
is provably convex only in some special cases, such as when $G$ is tree structured,}
we find in practice that this approximation gives very accurate solutions, even on general loopy graphs where its convergence is no longer theoretically guaranteed.

\qiangnew{
The global convergence guarantees of the proximal point algorithm may also fail if the inner update \eqref{equ:proximal_inner} is not solved exactly. 
It should also be possible to develop globally convergent algorithms without inner loops using the techniques that have been developed for full marginalization or MAP problems  \citep[e.g.,][]{meltzer2009convergent, hazan2010norm, jojic2010accelerated, SavchynskyySmooth12}, but we leave this to future work. }




%
\section{ Connections to EM}
\label{sec:EM}
A natural algorithm for solving the marginal MAP problem is to use the expectation-maximization (EM) algorithm,
by treating $\vx_A$ as the hidden variables and $\vx_B$ as the ``parameters'' to be maximized.  In this
section, we show that the EM algorithm can be seen as a coordinate ascent algorithm on a mean field variant
of our framework.  

We start by introducing a ``non-convex" generalization of Theorem~\ref{thm:duality}.  
\begin{cor} 
\label{cor:nonconvex_mixduality}
Let $\margpolyNull$ be the subset of the marginal polytope $\margpoly$ corresponding to the distributions in which $\vx_B$ are clamped to some deterministic values, that is, 
$$\margpolyNull = \{ \vtau \in  \margpoly~ \colon ~ \text{$\exists \vx_B^* \in \X_{B}$, such that $\tau(\vx_B) =  \deltaIndic(\vx_B = \vx_B^*) $}  \}.$$ 
Then the dual optimization \eqref{equ:mixduality} remains exact if the marginal polytope $\margpoly$ is replaced by any $\N$ satisfying $\margpolyNull  \subseteq \N \subseteq  \margpoly$, that is, 
\begin{align}
\label{equ:nonconvex_mixdaulity}
\Phi_{AB}  = \max_{\vtau \in \N}  \{  \langle \vtheta, \vtau \rangle + \HABtau \}.
\end{align}
\end{cor}
\begin{proof}
For an arbitrary marginal MAP solution $\vx_B^{*}$, the $\vtau^{*}$ with ${\qtau^{*}}(\vx) = p(\vx | \vx_B = \vx_B^{*}; \vtheta)$ 
is an optimum of \eqref{equ:mixduality} and satisfies $\vtau^{*} \in \margpolyNull $. Therefore, restricting the optimization on $\margpolyNull$ (or any $\N$) does not change the maximum value of the objective function. 
\end{proof}
\textbf{Remark.} Among all $\N$ satisfying $\margpolyNull  \subseteq \N \subseteq  \margpoly$, the marginal polytope $\margpoly$ is the smallest (and the unique) convex
set that includes $\margpoly^o$, i.e., it is the convex hull of $\margpoly^o$.

To connect to EM, we define $\margpoly^{\times}$,  the set of
distributions in which $\vx_A$ and $\vx_B$ are independent, that is,  
$\margpoly^{\times} = \{\vtau \in \margpoly \colon 
\qtau(\vx) =\qtau(\vx_A) \qtau(\vx_B) \}$.
Since $\margpolyNull \subset \margpoly^{\times} \subset \margpoly$, the dual optimization \eqref{equ:mixduality} remains exact when restricted to $\margpoly^\times$, that is, 
\begin{align}
\Phi_{AB}(\vtheta) = \max_{\vtau \in \M^{\times}}  \{ \langle \vtheta, \vtau \rangle  + \HABtau \} = \max_{\vtau \in \M^{\times}}  \{ \langle \vtheta, \vtau \rangle  + \HAtau \}, 
\end{align}
where the second equality holds because  $\HABtau = \HAtau$ for $\vtau \in \M^{\times}$. 

Although $\margpoly^{\times}$ is no longer a convex set, it is natural to consider a coordinate update that alternately optimizes $\tau(\vx_A)$ and $\tau(\vx_B)$, 
\begin{equation}\begin{split}
\text{Updating sum part}:~~~~~~&  \vtau_A^{t+1} \gets \arg \!\!\!\! \max_{\vtau_A \in \margpoly_{A}}   \{ \langle\E_{{\qtau_B^{t}}}(\vtheta), \vtau_A \rangle +\HAtauA \} ,  \\
\text{Updating max part}: ~~~~~~& \vtau_B^{t+1} \gets \arg \!\!\!\! \max_{\vtau_B \in \margpoly_{B}} \langle\E_{\qtau_A^{t+1}}(\vtheta ), \vtau_B \rangle     ,   
\label{equ:dualEM}
\end{split}\end{equation}
where $\M_A$ and $\M_B$ are the marginal polytopes over $\vx_A$ and $\vx_B$, respectively.  
%
Note that the sum and max step each happen to be the dual of a sum-inference and max-inference problem, 
respectively. If we go back to the primal, and update the primal configuration 
$\vx_B$ instead of $\vtau_B$, \eqref{equ:dualEM} can be rewritten into
%
\begin{equation*}
\begin{split}
\text{E step}:~~~~~~&\tau_{A}^{t+1}(\vx_A)  \gets    p(\vx_A | {\vx}_B^{t}; \vtheta) , \\
\text{M step}:~~~~~~&\vx^{t+1}_B   \gets \arg \max_{\vx_B} \E_{\tau_A^{t+1}}(\vtheta),
\end{split}
\end{equation*}
which is exactly the EM update, viewing $\vx_B$ as parameters and $\vx_A$ as hidden variables. Similar connections between EM and the coordinate ascent method on variational objectives has been discussed in \citet{Neal98} and \citet{Wainwright08}.

When the E-step or M-step are intractable, one can insert various
approximations. In particular, approximating $\margpoly_A$ by a
mean-field inner bound $\margpoly_A^{mf}$ leads to variational EM.  An interesting
observation is obtained by using a Bethe approximation \eqref{equ:bethe} to solve the E-step and a 
linear relaxation to solve the M-step; in this case, the EM-like update is equivalent to
solving
\begin{equation}
\max_{\vtau \in \localpoly^{\times}} \big \{  \langle \vtheta, \vtau \rangle   +  \sum_{i\in A}{\Hitau}  \ - \sum_{(ij)\in E_A} \Iijtau  \big \}, 
\label{equ:EMenergy}
\end{equation}
where $\localpoly^{\times}$ is the subset of $\localpoly$ in which $\tau_{ij} (x_i,
x_j) = \tau_i(x_i ) \tau_j (x_j)$ for $(ij)\in \cross$. Equivalently,
$\localpoly^{\times}$ is the subset of $\localpoly$ in which $\Iijtau = 0$ for $(ij)\in
\cross$. Therefore, \eqref{equ:EMenergy} can be treated as a special case of
\eqref{equ:Phitree} by taking $\rho_{ij} \to +\infty$, forcing the solution
$\tau^*$ to fall into $\localpoly^{\times}$.  As we discussed in Section~\ref{sec:globalopt}, EM represents an extreme of the tradeoff between convexity and integrality implied by Theorem~\ref{thm:betheglobalopt}, which strongly encourages vertex solutions by sacrificing convexity, and hence is likely to become stuck in local optima.

\section{Junction Graph Belief Propagation for Marginal MAP}
\label{sec:junctiongraph}
In the above,  we have restricted the discussion to pairwise models and pairwise entropy approximations, mainly for the purpose of clarity. In this section, we extend our algorithms to leverage higher order cliques, based on the junction graph representation  \citep{mateescu2010join, Koller_book}. 
Other higher order methods, like generalized BP \citep{Yedidia_Bethe} or  their convex variants \citep{Wainwright_TRBP, Wiegerinck05}, can be derived similarly. 


\newcommand{\Jmaxset}{\pi}
\newcommand{\JG}{\mathcal{G}}
\newcommand{\JV}{\mathcal{V}}
\newcommand{\JE}{\mathcal{E}}
\newcommand{\JC}{\mathcal{C}}
\newcommand{\JS}{\mathcal{S}}
\newcommand{\JA}{\mathcal{A}}
\newcommand{\JB}{\mathcal{B}}
\newcommand{\entcolon}{\colon}
\renewcommand{\L}{\mathbb{L}(\JG)}
\newcommand{\pa}[1]{{\mathrm{pa}({#1})}}
\newcommand{\JDec}[0]{\mathcal{D}}
\newcommand{\JChan}[0]{\mathcal{R}}
\renewcommand{\v}[1]{\boldsymbol{#1}}
\newcommand{\bpa}[1]{{\pi(#1)}}
\newcommand{\strong}{consistent }
\newcommand{\bdelta}[0]{\mathcal{\boldsymbol{\delta}}}
\newcommand{\EU}[0]{\mathrm{EU}}
\newcommand{\meu}{\mathrm{MEU}}
\newcommand{\MEU}{\mathrm{MEU}}
\newcommand{\Hcktau}{H_{c_k}(\vtau)}
\newcommand{\Hskltau}{H_{s_{kl}}(\vtau)}
\newcommand{\Hbktau}{H_{\Jmaxset_{k}}(\vtau)}
\newcommand{\Hcbktau}{H_{c_k | \Jmaxset_{k}}(\vtau)}

 %
For notation, a cluster graph is a graph of subsets of variables (called clusters). Formally, it is a triple $(\JG, \JC, \JS)$, where $\JG =
(\JV, \JE)$ is an undirected graph, with each node $k \in \JV$ associated with
a cluster $c_k \in \JC$, and each edge $(kl) \in \JE$ with a
subset $s_{kl}  \in \JS$  (called separators) satisfying $s_{kl} \subseteq c_k \cap c_l$.
We assume that $\JC$ subsumes the index set $\mathcal{I}$, that is, for any
$\alpha \in \mathcal{I}$, we can assign it with a $c_k \in \JC$, denoted $c[\alpha]$,
such that $\alpha \subseteq c_k$.
 In this case, we can reparameterize $\vtheta
= \{ \theta_{\alpha} \colon \alpha \in \mathcal{I}\} $ into  $\vtheta = \{
\theta_{c_k} \colon k \in \JV \} $ by taking 
$\displaystyle \theta_{c_k} = \!\!\!\!\!\sum_{\alpha \colon c[\alpha] = c_k}  \!\!\!\!\theta_{\alpha}$, 
without changing the distribution. 
Therefore, we simply assume $\JC = \I$ in this paper without loss of generality. 
%
A cluster graph is called a \emph{junction graph} if it satisfies the
\emph{running intersection property} -- for each $i \in V$, the induced
sub-graph consisting of the clusters and separators that include $i$ is a
connected tree. A junction graph is a junction tree if $\JG$ is a tree. 

To approximate the variational dual form, we first replace $\M$ with a higher order 
locally consistent polytope $\L$, which is the set of local marginals $\vtau = \{ \tau_{c_k}, \tau_{s_{kl}} \colon k \in \JV, (kl)\in \JE \}$ that are consistent on the intersections of the clusters and separators, that is,  
$$\L = \{ \vtau \colon   \sum_{x_{c_k \setminus s_{kl}}}\tau_{c_k}(x_{c_k}) = \tau(x_{s_{kl}}), \tau_{c_k}(x_{c_k}) \geq 0, \text{for $\forall ~ k\in \JV, (kl)\in \JE$} \}.  $$
Clearly, we have $\M \subseteq \L$ and that $\L$ is tighter than the pairwise polytope $\localpoly$ we used previously. 

We then approximate the joint entropy term by a linear combination of the entropies over the clusters and separators, 
\begin{align*}
H(\vtau)  \approx  \sum_{k \in \JV}\Hcktau - \sum_{(kl) \in \JE} \Hskltau, 
\end{align*}
where $\Hcktau$ and $\Hskltau$ are the entropy of the local marginals $\tau_{c_k}$ and $\tau_{s_{kl}}$, respectively. 
Further, we approximate $\HBtau$ by a slightly more restrictive entropy decomposition, 
$$
\HBtau \approx \sum_{k \in \JV}  H_{\Jmaxset_k}(\vtau), 
$$
where $\{\Jmaxset_k  \colon k \in \JV \}$ is a non-overlapping partition of the max nodes $B$ satisfying $\Jmaxset_k \subseteq c_k$ for $\forall k \in \JV$.
In other words, $\Jmaxset$ represents an assignment of each max node $x_b \in B$ into a cluster $k$ with $x_b \in \Jmaxset_k$.
Let $\JB$ be the set of clusters $k \in \JV$ for which $\Jmaxset_k \neq \emptyset$, 
and call $\JB$ the \emph{max-clusters}; correspondingly, call $\JA = \JV \setminus \JB$ the \emph{sum-clusters}. See \figref{fig:junction_example} for an example.
%
\begin{figure*}[t]     
\begin{tabular}{ccc}
 \qquad \includegraphics[height= .25\textwidth]{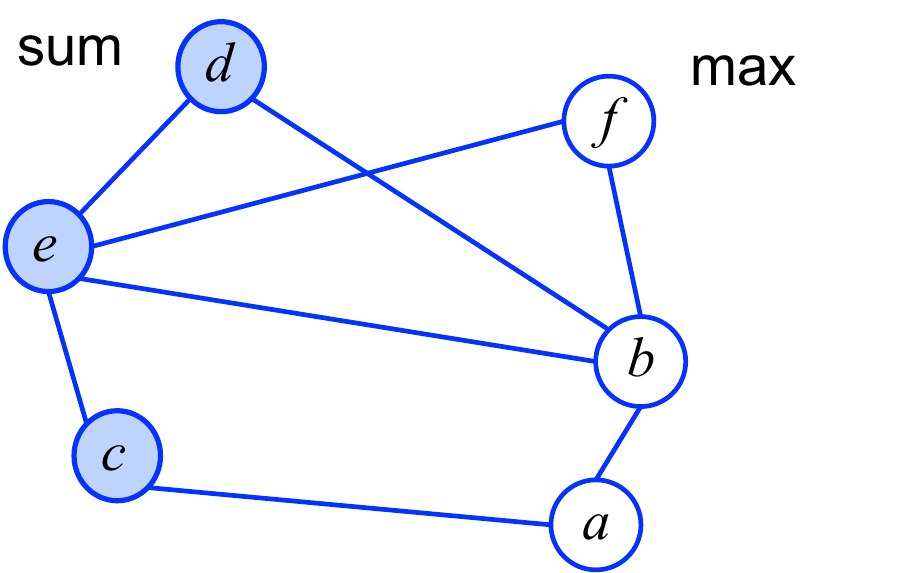}    &  &
\raisebox{1em}{\includegraphics[height= .2\textwidth]{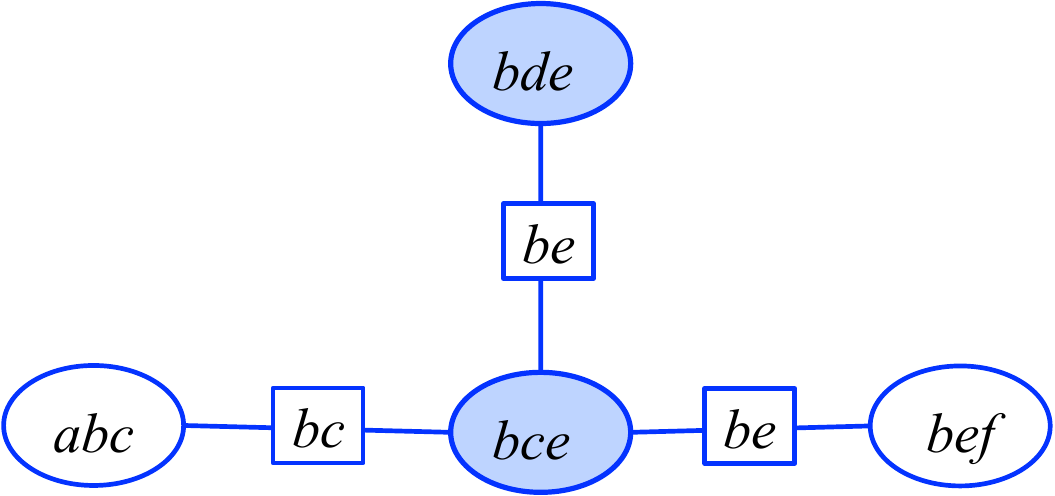}} \\
  {\small (a) } & &  {\small (b) }
\end{tabular}
\caption{(a) An example of marginal MAP problem, where $d,c,e$ are sum nodes (shaded) and $a, b, f$ are max nodes.  
(b) A junction graph of (a).  Selecting a partitioning of max nodes, $\Jmaxset_{bde}=\Jmaxset_{bef}=\emptyset$, $\Jmaxset_{abc} = \{a,b\}$, and $\Jmaxset_{bef}=\{f\}$,
results in $\{bde\}, \{bce\}$ being sum clusters (shaded) and $\{abc\}, \{bef\}$ being max clusters.
%
}
\label{fig:junction_example}
\end{figure*}

Overall, the marginal MAP dual form in \eqref{equ:mixduality} is approximated by 
\begin{align}
\max_{\vtau \in \L} \big\{   \langle \vtheta , \vtau \rangle +  \sum_{k \in \JA}  \Hcktau   +   \sum_{k \in \JB} \Hcbktau -  \sum_{(kl )\in \JE} \Hskltau \big\}
\label{equ:jgraphdualobj}
\end{align}
where $\Hcbktau = \Hcktau - \Hbktau$. 
Optimizing \eqref{equ:jgraphdualobj} using a method similar to the derivation of mixed-product BP in Algorithm~\ref{alg:mix_product_msg}, we obtain a  ``mixed-product" junction graph belief propagation, given in Algorithm~\ref{alg:mix_jgraph_BP}. 
%
%
\begin{algorithm}[tb]
\caption{Mixed-product Junction Graph BP} 
\label{alg:mix_jgraph_BP}
\begin{algorithmic}
\STATE 1. Passing messages  between clusters on the junction graph until convergence: 
\begin{align*}
&\begin{tabular}{c}{$\JA \to \JA \cup \JB$}:\\ {(\small sum-product)}\end{tabular} && m_{k \to l}(x_{s_{kl}}) \propto  \sum_{x_{c_k \setminus s_{kl}} } \psi_{c_k}(x_{c_k}) m_{\sim k \setminus l}(x_{c_k})  ,  \\
&\begin{tabular}{c}{$\JB \to \JA \cup \JB$}:\\ {(\small argmax-product)}\end{tabular}  && m_{k \to l}(x_{s_{kl}})  \propto \sum_{x_{c_k \setminus s_{kl}} }( \psi_{c_k}(x_{c_k})  m_{\sim k \setminus l}(x_{c_k})  ) \cdot \deltaIndic[x_{\Jmaxset_k} \in \X^*_{\Jmaxset_k}],  \\
&&& \hspace{-.2\textwidth}  \text{where $\X^*_{\Jmaxset_k} = \argmax_{x_{\Jmaxset_k}}  \sum_{x_{c_k \setminus \Jmaxset_k}} b_k(x_{c_k})$, } \\
&&&  \hspace{-.27\textwidth}   \text{\ \ \ \ \  \ \ \ $b_k(x_{c_k}) =  \psi_{c_k}(x_{c_k}) \!\!\! \prod_{k' \in \mathcal{N}(k)} \!\!\! m_{k' \to k}(x_{s_{k'k}})$~~  and ~~$m_{\sim k \setminus l}(x_{c_k}) ~ = \!\!\!\!\!\!\! \prod_{k' \in \mathcal{N}(k)\setminus \{l\}}  \!\!\! \!\!\! m_{k'\to k}(x_{s_{k'k}})$.}
\end{align*}
\STATE 2. Decoding: 
$\displaystyle \vx^*_{\Jmaxset_k} = \argmax_{x_{\Jmaxset_k}}  \sum_{x_{c_k \setminus \Jmaxset_k}} b_k(x_{c_k})$
for $\forall k\in \JB$. 
\end{algorithmic}
\end{algorithm} 

Similarly to our mixed-product BP in Algorithm~\ref{alg:mix_product_msg}, Algorithm~\ref{alg:mix_jgraph_BP} also admits an intuitive reparameterization interpretation and a strong local optimality guarantee. 
Algorithm~\ref{alg:mix_jgraph_BP} can be seen as a special case of a more general junction graph BP algorithm derived in \citet{liu12b} for solving maximum expected
utility tasks in decision networks.  For more details, we refer the reader to that work.

\section{Experiments}
\label{sec:experiments}
We illustrate our algorithms on both simulated models and more realistic diagnostic Bayesian networks taken from the UAI08 inference challenge. 
We show that our Bethe approximation algorithms perform best among all the tested algorithms, including \citet{Jiang10}'s hybrid message passing and a state-of-the-art local search algorithm \citep{Park04}. 

We implement our mixed-product BP in Algorithm~\ref{alg:mix_product_msg} with Bethe weights ({\tt mix-product (Bethe)}), the regular sum-product BP ({\tt sum-product}), max-product BP ({\tt max-product}) and \citet{Jiang10}'s hybrid message passing (with Bethe weights) in Algorithm~\ref{alg:jiang_product} ({\tt Jiang's method}), where the solutions are all extracted by maximizing the singleton marginals of the max nodes. For all these algorithms, we run a maximum of 50 iterations; in case they fail to converge, we run 100 additional iterations with a damping coefficient of $0.1$. We initialize all these algorithms with 5 random initializations and pick the best solution; for {\tt mix-product (Bethe)} and {\tt Jiang's method}, we run an additional trial initialized using the sum-product messages, which was reported to perform well in \citet{Park04} and \citet{Jiang10}. We also run the proximal point version of mixed-product BP with Bethe weights ({\tt Proximal (Bethe) }), which is Algorithm~\ref{alg:proximal_point} with both $\HAB$ and $\HB$ approximated by Bethe approximations. 

We also implement the TRW approximation, but only using the convergent proximal point algorithm, because the TRW upper bounds are valid only when the algorithms converge. 
The TRW weights of $\hat{H}_{A|B}$ are constructed by first (randomly) selecting spanning trees  of $G_A$, and then augmenting each spanning tree with one uniformly selected edge in $\partial_{AB}$; 
the TRW weights of $\hat{H}_B(\vtau)$ are constructed to be provably convex, using the method of TRW-S in \citet{trws}. We run all the proximal point algorithms for a maximum of 100 iterations, with a maximum of 5 iterations of weighted message passing updates \eqref{equ:weightedmsg}-\eqref{equ:weighted_marginals} for the inner loops (with 5 additional damping with 0.1 damping coefficient). 
   
In addition, we compare our algorithms with SamIam, which is a state-of-the-art implementation of the local search algorithm for marginal MAP \citep{Park04}; we use its default Taboo search method with a maximum of 500 searching steps, and report the best results among 5 trials with random initializations, and one additional trial initialized by its default method (which sequentially initializes $x_i$ by maximizing $p(x_{i} | x_{\mathrm{pa}_i})$ along some predefined order). 

We also implement an EM algorithm, whose expectation and maximization steps are approximated by sum-product and max-product BP, respectively.  We run EM with 5 random initializations and one initialization by sum-product marginals, and pick the best solution.

\textbf{Simulated Models.} 
%
%
We consider pairwise models over discrete random variables taking values in $\{-1,0, +1\}^n$, 
\begin{equation*}
p(\vx) \propto \exp\big[\sum_i \theta_{i}(x_i) + \sum_{(ij)\in E} \theta_{ij}(x_{i}, x_{j})\big].
\end{equation*}
The value tables of $\theta_i$ and $\theta_{ij}$ are randomly generated from normal distribution, $\theta_{i}(k) \sim \mathrm{Normal}(0, 0.01)$, $\theta_{ij}(k,l) \sim \mathrm{Normal}(0, \sigma^2)$, where $\sigma$ controls the strength of coupling. Our results are averaged on 1000 randomly generated sets of parameters. 

We consider different choices of graph structures and max / sum node patterns: 
\begin{enumerate}
\item \emph{Hidden Markov chain} with 20 nodes, as shown in \figref{fig:hiddenchain}.  
\item \emph{Latent tree models}. We generate random trees of size 50, by finding the minimum spanning trees of random symmetric matrices with elements drawn from $\mathrm{Uniform}([0,1])$. We take the leaf nodes to be max nodes, and the non-leaf nodes to be sum nodes.  See \figref{fig:rand_tree_result}(a) for a typical example. 
\item \emph{$10\times10$ Grid} with max and sum nodes distributed in two opposite chess board patterns shown in \figref{fig:chessboard_result}(a) and \figref{fig:chessboard_rev_result}(a), respectively.  In \figref{fig:chessboard_result}(a), the sum part is a loopy graph, and the max part is a (fully disconnected) tree; in  \figref{fig:chessboard_rev_result}(a), the max and sum parts are flipped. 
\end{enumerate}

The results on the hidden Markov chain are shown in \figref{fig:hiddenchain_result}, where we plot in panel (a) different algorithms' percentages of obtaining the globally optimal solutions among 1000 random trials, 
and in panel (b) their relative energy errors defined by $Q(\hat{\vx}_B; \vtheta) - Q(\vx_B^*; \vtheta)$, where $\hat{\vx}_B$ is the solution returned by the algorithms, and $\vx_B^*$  is the true optimum. 

The results of the latent tree models and the two types of 2D grids are shown in \figref{fig:rand_tree_result},  \figref{fig:chessboard_result} and \figref{fig:chessboard_rev_result}, respectively. Since the globally optimal solution $\vx_B^*$ is not tractable to calculate in these cases, we report the approximate relative error defined by $Q(\hat{\vx}_B; \vtheta) - Q(\tilde{\vx}_B; \vtheta)$, where $\tilde{\vx}_B$ is the best solution we found across all algorithms.


\textbf{Diagnostic Bayesian Networks.} 
We also test our algorithms on two diagnostic Bayesian networks taken from the UAI08 Inference Challenge, where we construct marginal MAP problems by randomly selecting varying percentages of nodes to be max nodes. Since these models are not pairwise, we implement the junction graph versions of {\tt mix-product (Bethe)} and {\tt proximal (Bethe)} shown in Section~\ref{sec:junctiongraph}.  \figref{fig:uai_result} shows the approximate relative errors of our algorithms and {\tt local search (SamIam)} as the percentage of the max nodes varies. 

\textbf{Insights.}
Across all the experiments, we find that {\tt mix-product (Bethe)}, {\tt proximal (Bethe)} and {\tt local search (SamIam)} significantly outperform all the other algorithms, while {\tt proximal (Bethe)} outperforms the two others in some circumstances. In the hidden Markov chain example in \figref{fig:hiddenchain_result}, these three algorithms almost always (with probability $\geq 99 \%$) find the globally optimal solutions. However, the performance of SamIam tends to degenerate when the max part has loopy dependency structures (see \figref{fig:chessboard_rev_result}), or when the number of max nodes is large (see \figref{fig:uai_result}), both of which make it difficult to explore the solution space by local search. On the other hand, {\tt mix-product (Bethe)} tends to degenerate as the coupling strength $\sigma$ increases (see \figref{fig:chessboard_rev_result}), probably because its convergence gets worse as $\sigma$ increases. 

We note that our TRW approximation gives much less accurate solutions than the other algorithms, but is able to provide an upper bound on the optimal energy. Similar phenomena have been observed for TRW-BP in standard max- and sum- inference. 

The hybrid message passing of \citet{Jiang10} is significantly worse than {\tt mix-product (Bethe)}, {\tt proximal (Bethe)} and {\tt local search (SamIam)}, but is otherwise the best among the remaining algorithms. EM performs similarly to (or sometimes worse than) Jiang's method. 

The regular max-product BP and sum-product BP are among the worst of the tested algorithms, indicating the danger of approximating mixed-inference by pure max- or sum- inference. 
Interestingly, the performances of max-product BP and sum-product BP have opposite trends: In \figref{fig:hiddenchain_result}, \figref{fig:rand_tree_result} and \figref{fig:chessboard_result}, where the max parts are fully disconnected and the sum parts are connected and loopy, max-product BP usually performs worse than sum-product BP, but gets better as the coupling strength $\sigma$ increases; sum-product BP, on the other hand, tends to degenerate as $\sigma$ increases. In \figref{fig:chessboard_rev_result}, where the max / sum pattern is reversed 
(resulting in a larger, loopier max subgraph), max-product BP performs better than sum-product BP.

\begin{figure*}[t]     
\begin{tabular}{cc}
\!\!\!\!\!\!
\scalebox{0.95}{\includegraphics[width= .32\textwidth]{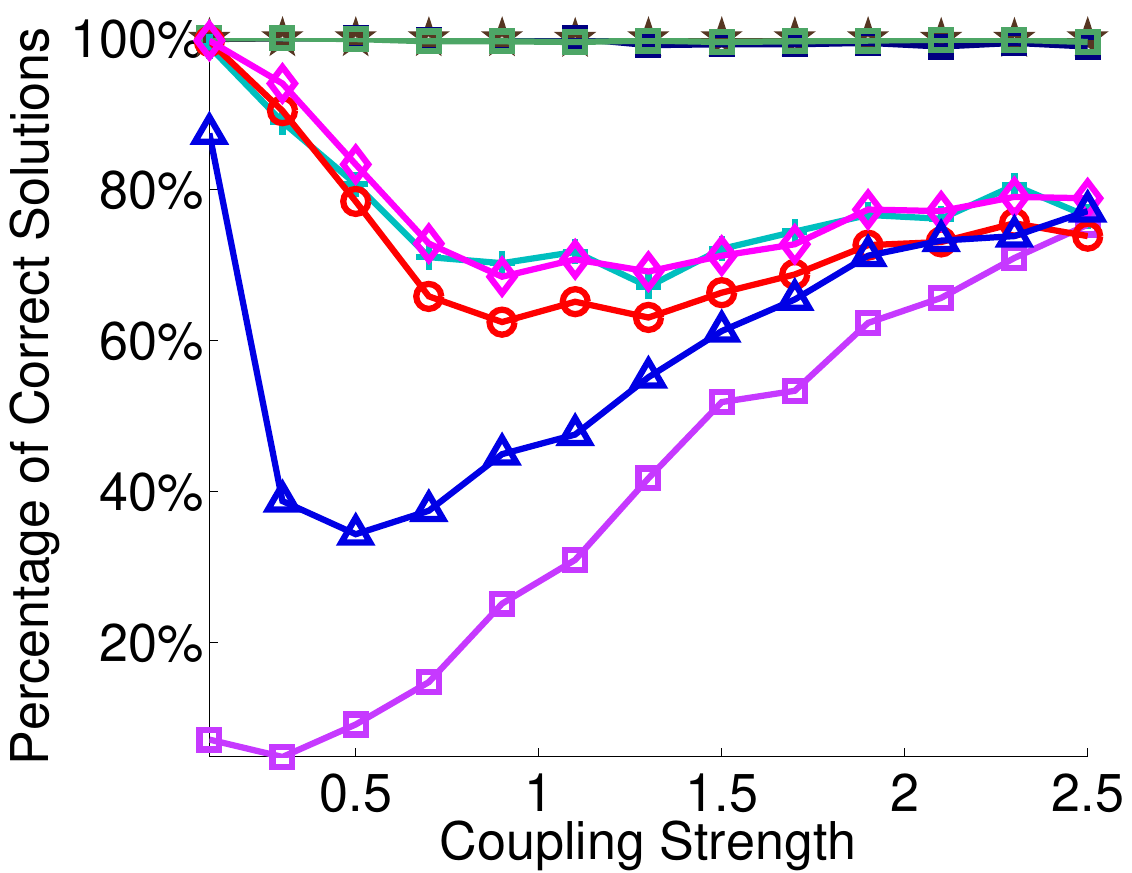}}  
 &
\hspace{2.8em}\scalebox{0.95}{\includegraphics[width= .32\textwidth]{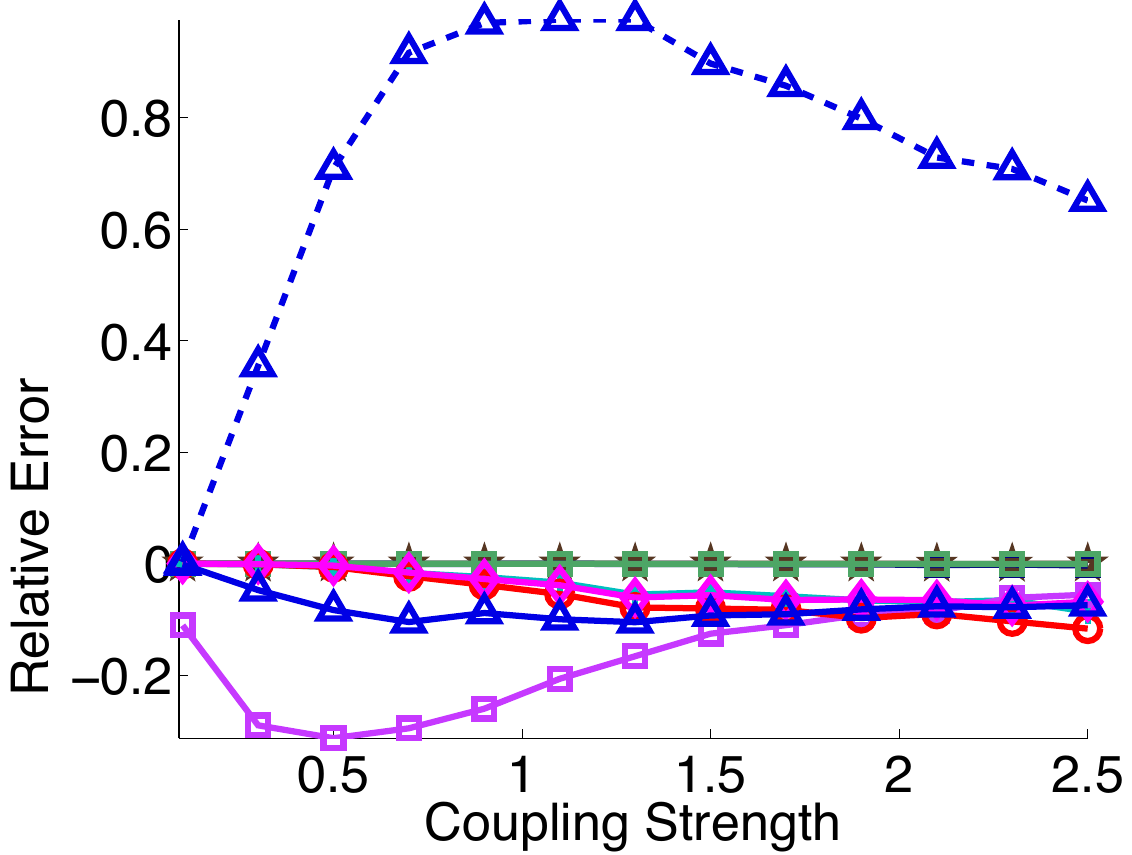} \qquad   
\hspace{-2.6em}\raisebox{.2em}{\begin{tikzpicture}
\shade[left color=gray!50!white,right color=gray!50!white] (0,0) -- (0,.8) -- (.6,2.8) -- (.6,-.4) -- cycle;
 \draw[gray!50!white]  (.6,-.4) -- (4.6,-.4) -- (4.6,2.8) -- (.6,2.8) -- cycle;
\end{tikzpicture}}
\hspace{-10.5em}\raisebox{.5em}{\includegraphics[width= .25\textwidth]{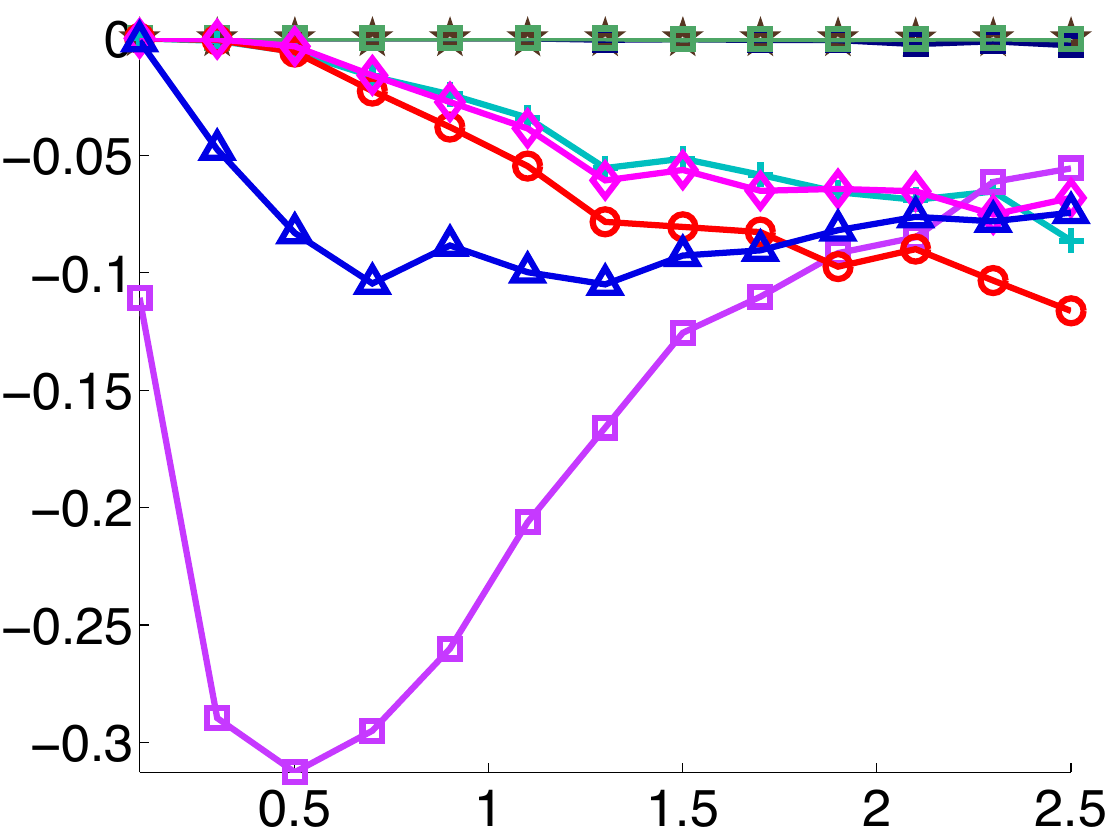}}}
   \\
    {\small (a) } & {\small (b)} 
\begin{picture}(0,0)
\put(-160,70){\includegraphics[width= .16\textwidth]{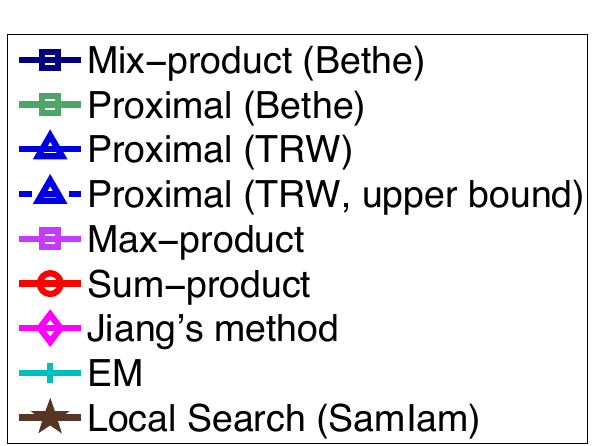}}
\end{picture}
\end{tabular}
\caption{Results on the hidden Markov chain in \figref{fig:hiddenchain} (best viewed in color). (a)
different algorithms' probabilities of obtaining the globally optimal solution among 1000 random trials. {\tt Mix-product (Bethe)}, {\tt Proximal (Bethe)} and {\tt Local Search (SamIam)} almost always (with probability $\geq 99\%$) find the optimal solution.   
(b) The relative energy errors of the different algorithms, and the upper bounds obtained by {\tt Proximal (TRW)} as a function of coupling strength $\sigma$.} 
\label{fig:hiddenchain_result}
\end{figure*}
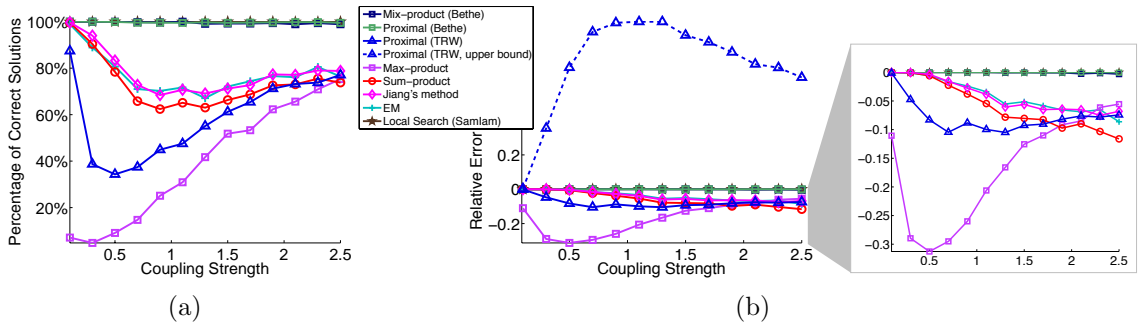
\begin{figure*}[t]     
\begin{tabular}{cc}
\!\!\!\!\!\!
\raisebox{0em}{\scalebox{0.9}{\includegraphics[width= .29\textwidth]{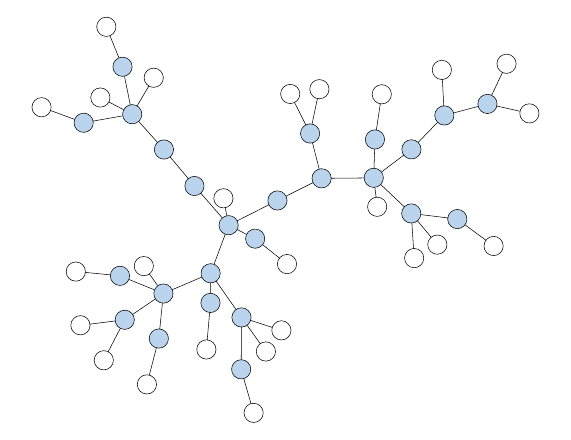}}}  
 &
\hspace{4em}\scalebox{0.95}{\includegraphics[width= .32\textwidth]{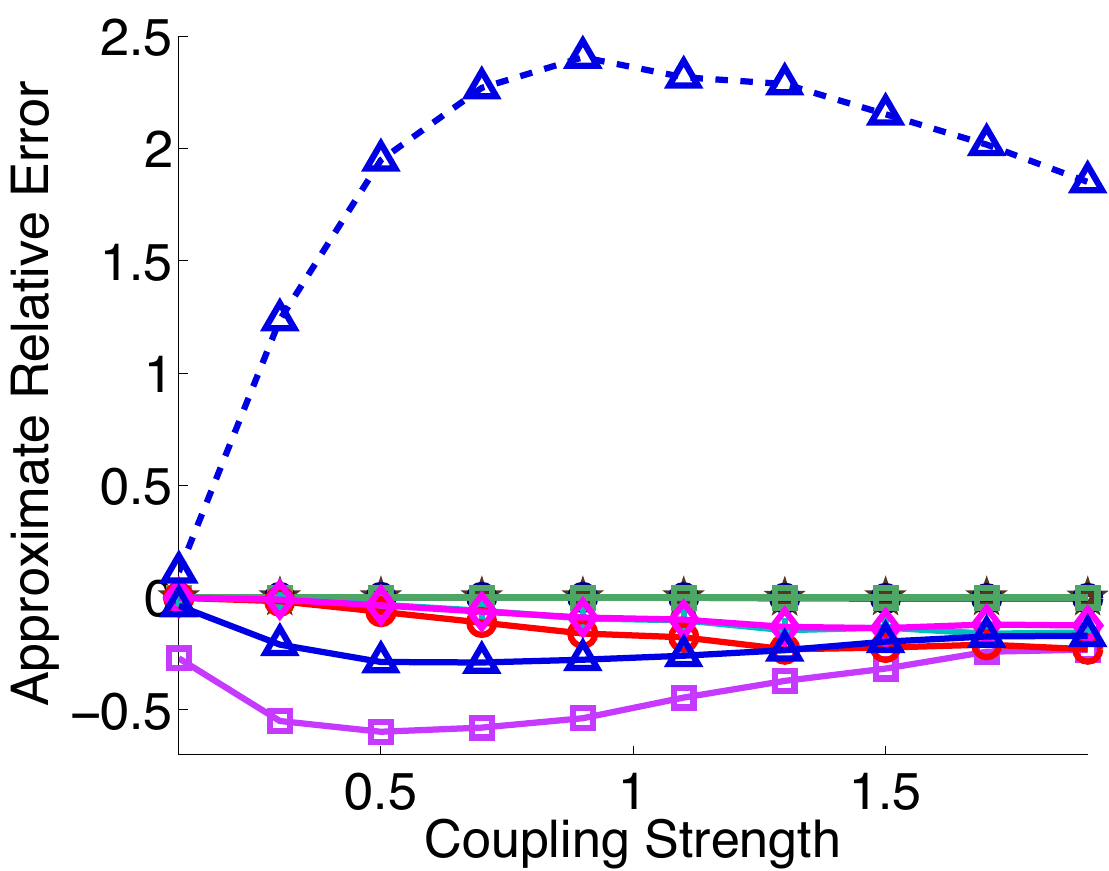} \qquad   
\hspace{-2.3em}\raisebox{.2em}{\begin{tikzpicture}
\shade[left color=gray!50!white,right color=gray!50!white] (0,0) -- (0,.8) -- (.6,2.8) -- (.6,-.4) -- cycle;
 \draw[gray!50!white]  (.6,-.4) -- (4.6,-.4) -- (4.6,2.8) -- (.6,2.8) -- cycle;
\end{tikzpicture}}
\hspace{-10.5em}\raisebox{.5em}{\includegraphics[width= .25\textwidth]{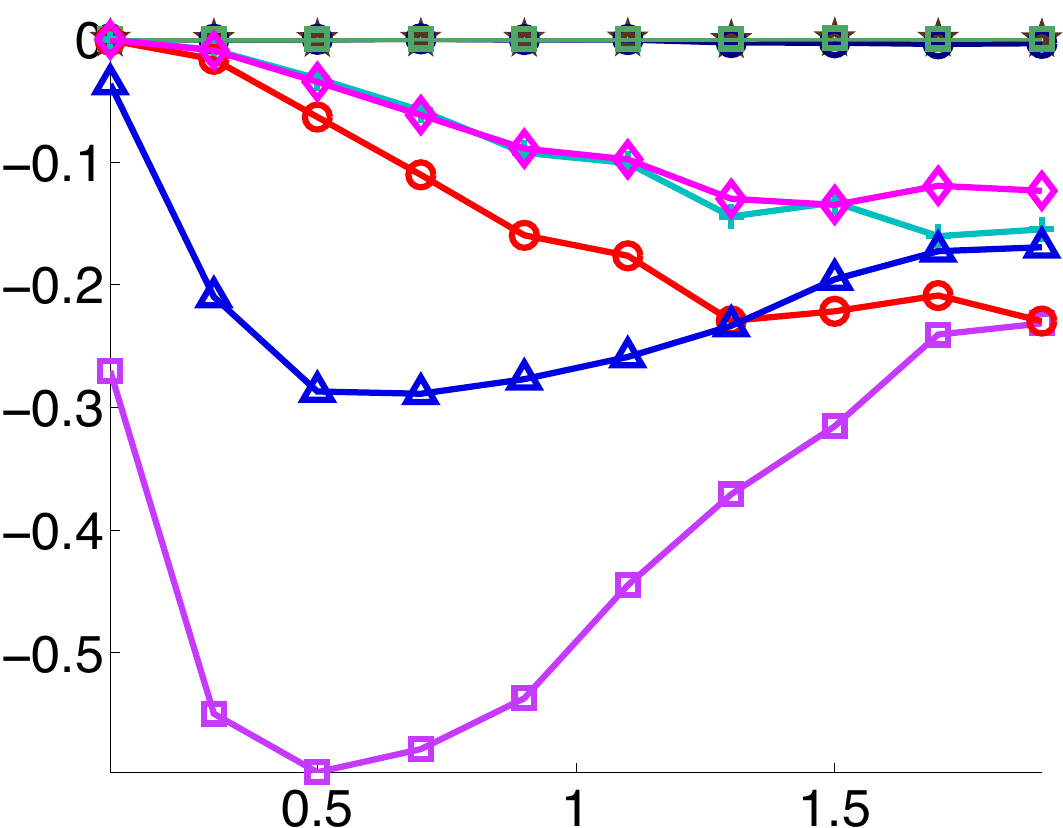}}}   \\
 {\small (a) } & {\small (b)} 
\begin{picture}(0,0)
\put(-182,20){\includegraphics[width= .16\textwidth]{figures_jmlr/hmm_state3_assym_LEGEND_withbound.pdf}}
\end{picture}
\end{tabular}
\caption{(a) A typical latent tree model, whose leaf nodes are taken to be max nodes (white) and non-leaf nodes to be sum nodes (shaded). 
(b) The approximate relative energy errors of different algorithms, and the upper bound
obtained by {\tt Proximal (TRW)} as a function of coupling strength $\sigma$.} 

\label{fig:rand_tree_result}
\end{figure*}
\begin{figure*}[tbh]     
\begin{tabular}{cc}
\raisebox{5.5em}{\begin{tabular}{c}
\raisebox{0em}{\scalebox{1}{\includegraphics[width= .15\textwidth]{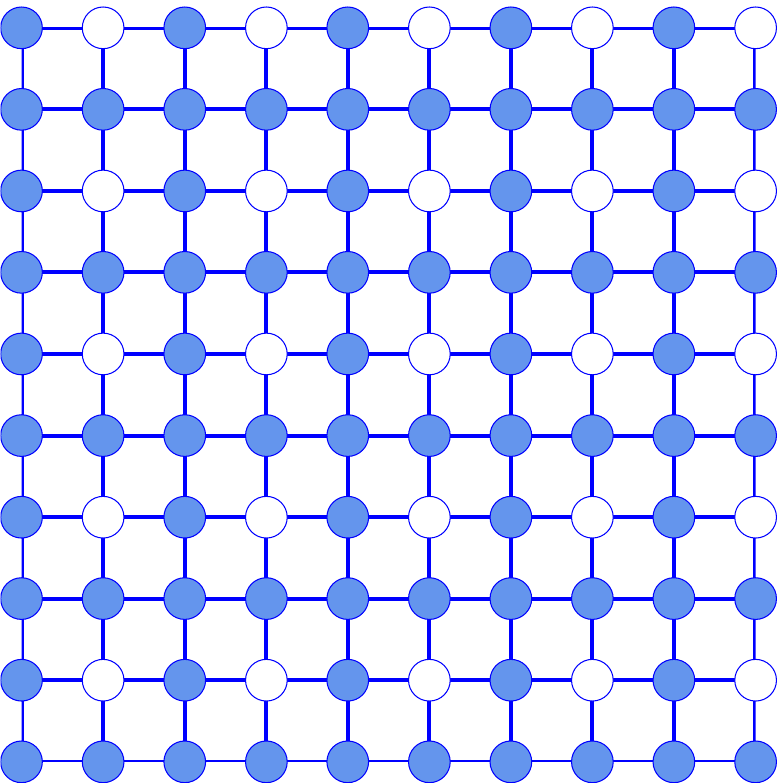}}}   \\
{\includegraphics[width= .15\textwidth]{figures_jmlr/hmm_state3_assym_LEGEND_withbound.pdf}}
\end{tabular}
}
 &
\hspace{0em}\scalebox{1.2}{\includegraphics[width= .32\textwidth]{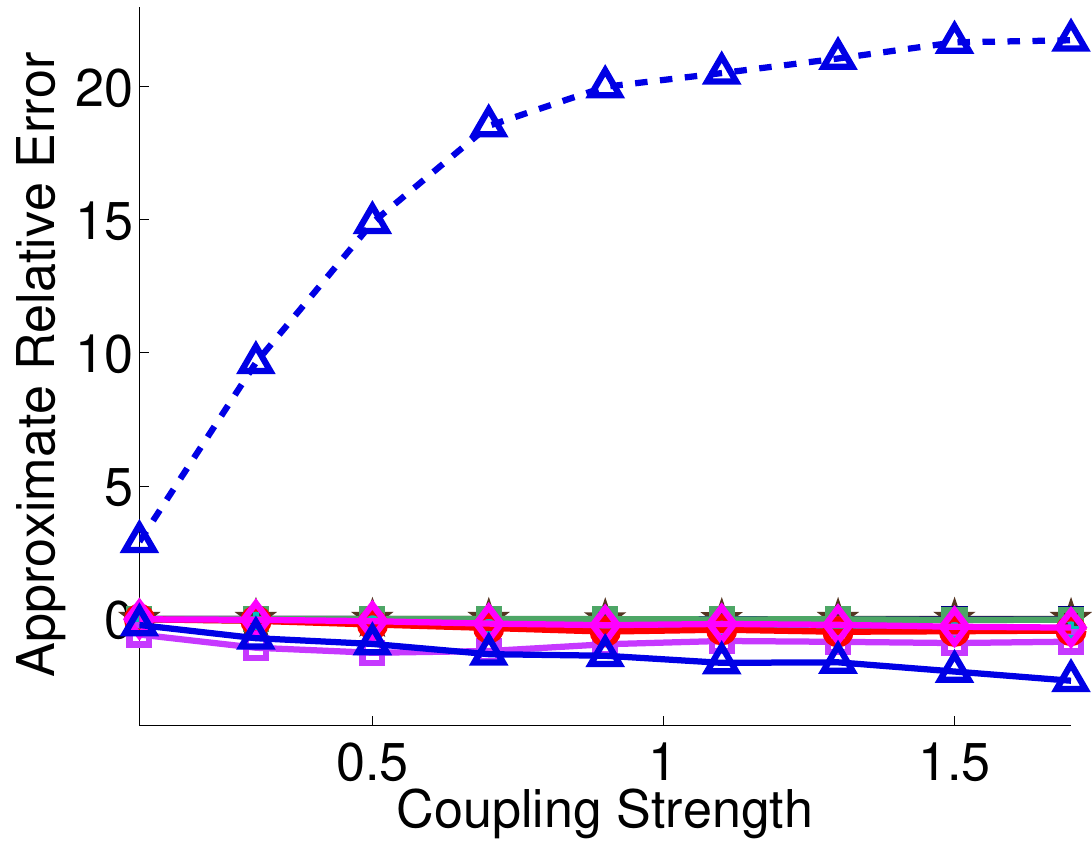} \qquad   
\hspace{-2.3em}\raisebox{.2em}{\begin{tikzpicture}
\shade[left color=gray!50!white,right color=gray!50!white] (0, 0) -- (0,.5) -- (.6,2.8) -- (.6,-.4) -- cycle;
 \draw[gray!50!white]  (.6,-.4) -- (4.6,-.4) -- (4.6,2.8) -- (.6,2.8) -- cycle;
\end{tikzpicture}}
\hspace{-10.5em}\raisebox{.8em}{\includegraphics[width= .25\textwidth]{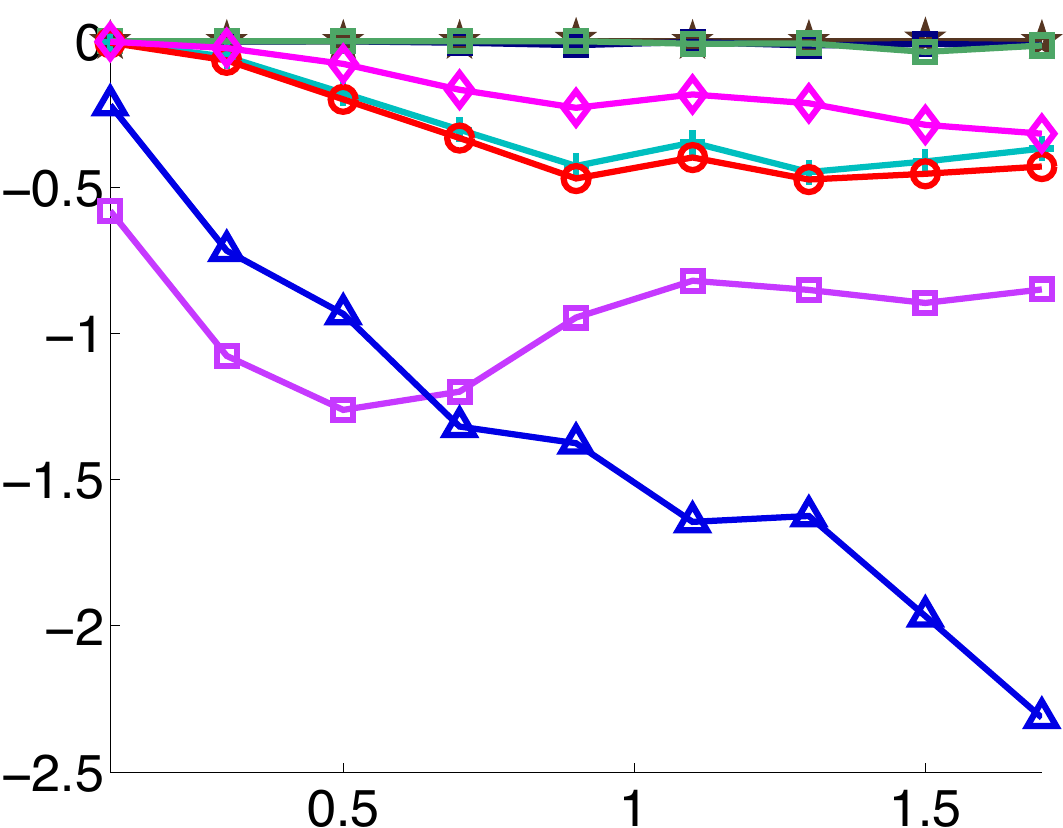}}}   \\
 {\small (a) } & {\small (b)} 
\end{tabular}
\caption{(a) A marginal MAP problem defined on a $10\times10$ Ising grid, with shaded sum nodes
and unshaded max nodes; note that the sum part is a loopy graph, while max part is fully disconnected. (b) The approximate relative errors
of different algorithms and the upper bound obtained by Proximal (TRW) as a function of coupling strength $\sigma$.} 
\label{fig:chessboard_rev_result} 
\label{fig:chessboard_result}
\end{figure*} 

\begin{figure*}[tbh]     
\begin{tabular}{cc}
\raisebox{5.5em}{\begin{tabular}{c}
\raisebox{0em}{\scalebox{1}{\includegraphics[width= .15\textwidth]{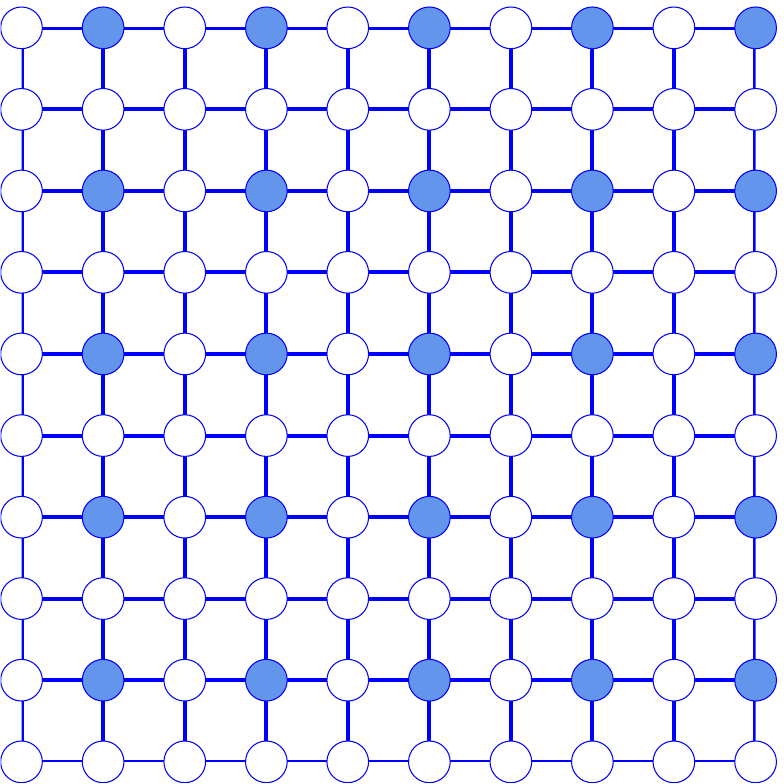}}}   \\
{\includegraphics[width= .15\textwidth]{figures_jmlr/hmm_state3_assym_LEGEND_withbound.pdf}}
\end{tabular}
}
 &
\hspace{0em}\scalebox{1.2}{\includegraphics[width= .32\textwidth]{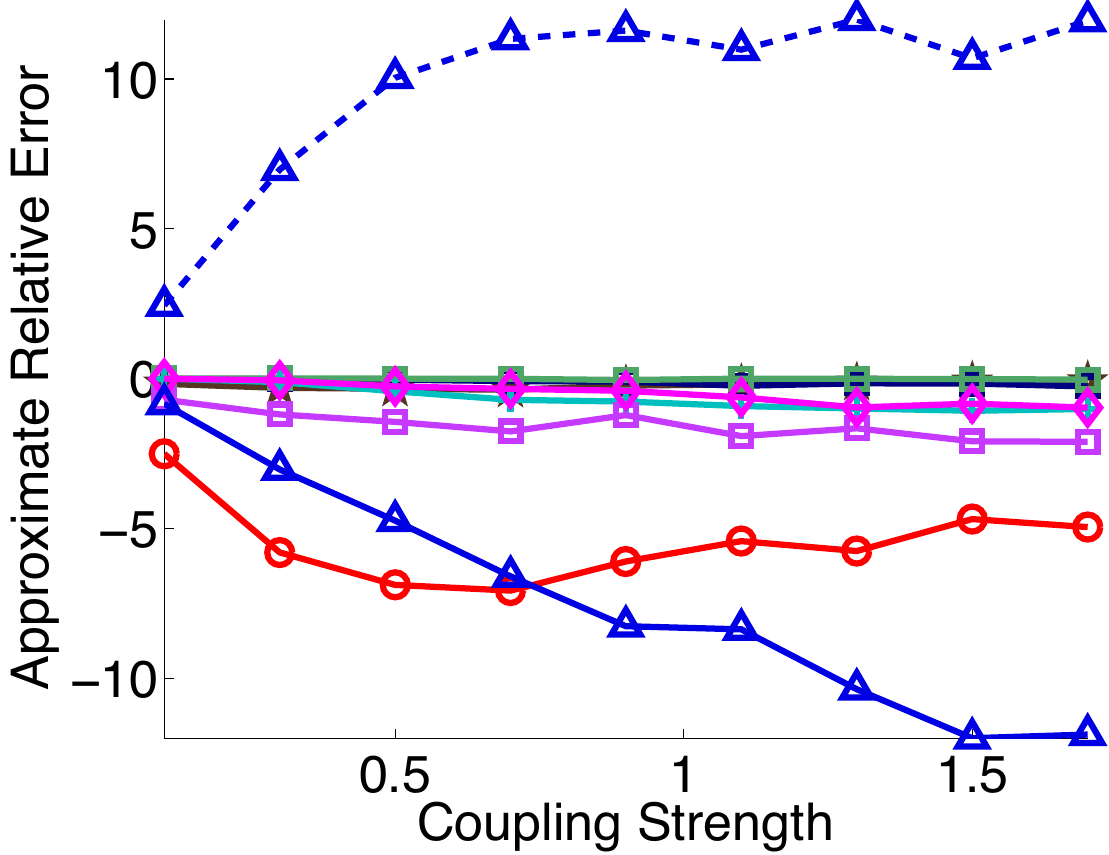} \qquad   
\hspace{-2.3em}\raisebox{.2em}{\begin{tikzpicture}
\shade[left color=gray!50!white,right color=gray!50!white] (0, 1.3) -- (0,1.65) -- (.6,2.8) -- (.6,-.4) -- cycle;
 \draw[gray!50!white]  (.6,-.4) -- (4.6,-.4) -- (4.6,2.8) -- (.6,2.8) -- cycle;
\end{tikzpicture}}
\hspace{-10.5em}\raisebox{.8em}{\includegraphics[width= .25\textwidth]{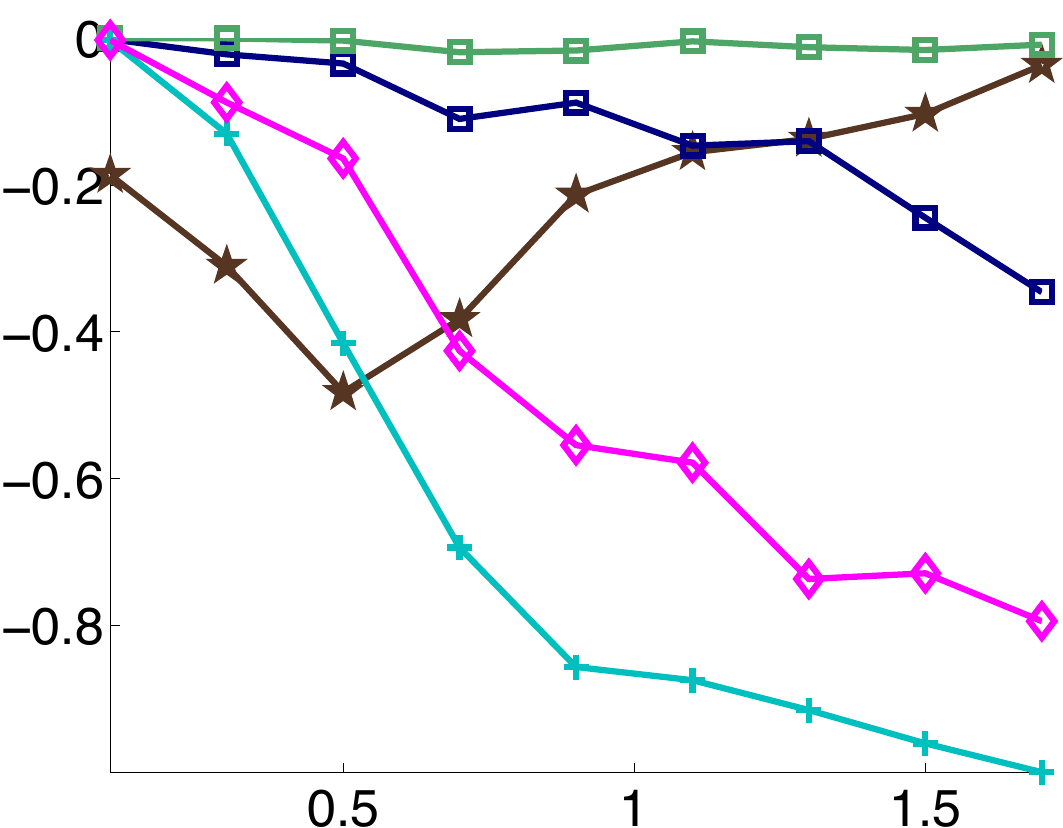}}}   \\
 {\small (a) } & {\small (b)} 
\end{tabular}
\caption{(a) A marginal MAP problem defined on a $10\times10$ Ising grid, but with max / sum part exactly opposite to that in \figref{fig:chessboard_result}; note that the max part is loopy, while the sum part is fully disconnected in this case. (b) The approximate relative errors
of different algorithms and the upper bound obtained by Proximal (TRW) as a function of coupling strength $\sigma$.} 
\label{fig:chessboard_rev_result}
\end{figure*}

\begin{figure*}[t]     
\begin{tabular}{c}
\hspace{-1em} \raisebox{0em}{\scalebox{1}{\includegraphics[width= 1\textwidth]{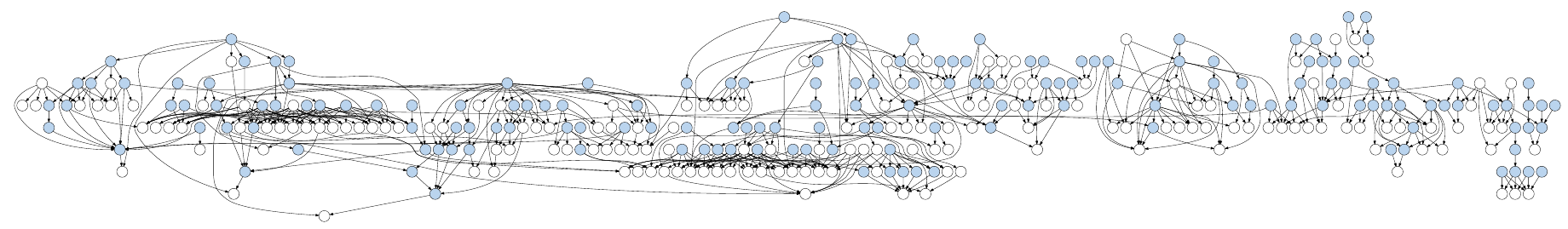}}}   \\
{\small (a) The structure of Diagnostic BN-2, with 50\% randomly selected sum nodes shaded. }
\end{tabular}
\begin{tabular}{cc}
\raisebox{0em}{\scalebox{1}{\includegraphics[width= .35\textwidth]{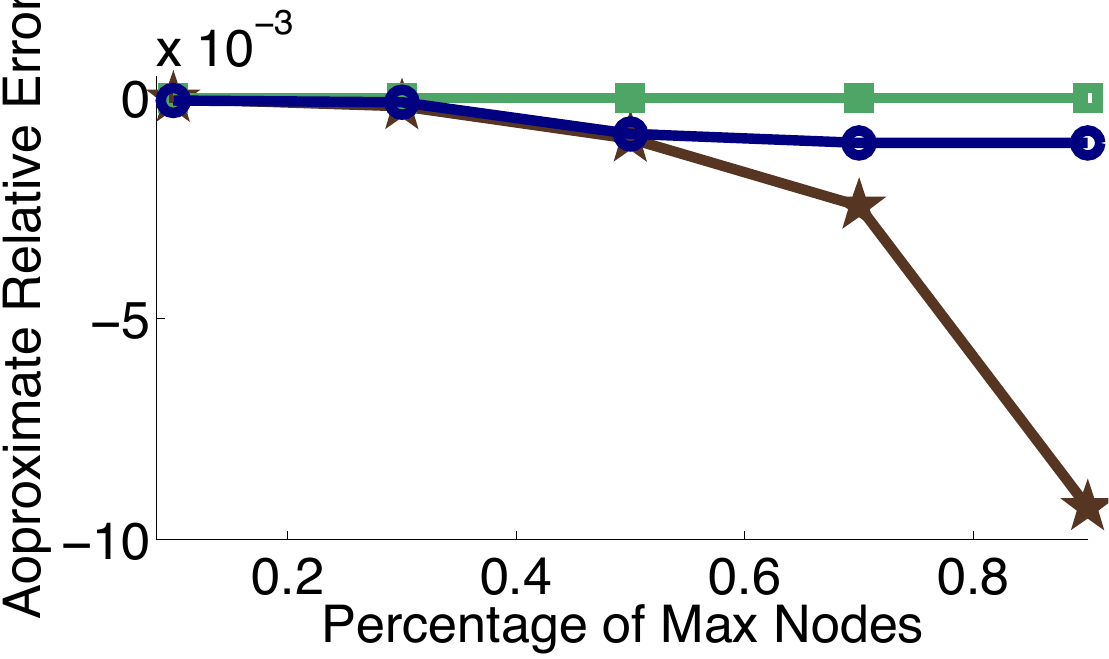}}}   &
\raisebox{0em}{\scalebox{1}{\includegraphics[width= .35\textwidth]{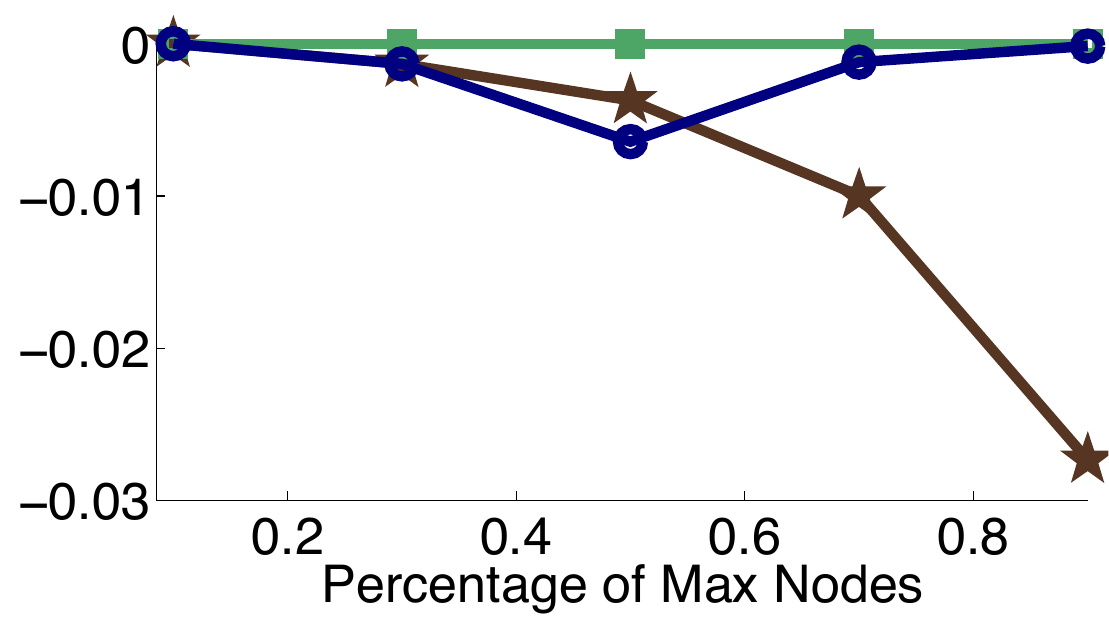}}}   \\
 {\small (b) Diagnostic BN-1 } & {\small (c) Diagnostic BN-2} 
\end{tabular}
\begin{picture}(0,0)
\put(0,20){\includegraphics[width= .2\textwidth]{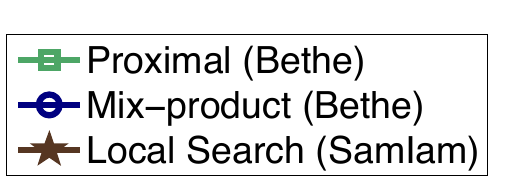}}
\end{picture}
\caption{The results on two diagnostic Bayesian networks (BNs) in the UAI08 inference challenge. (a) The Diagnostic BN-2 network.  (b)-(c) The performances of algorithms on the two BNs as a function of the percentage of max nodes. The local search method tends to degenerate when the number of max nodes is large, making it difficult to search over the solution space. 
Results are averaged over 100 random trials.} 
\label{fig:uai_result}
\end{figure*}

\section{Conclusion and Further Directions}
\label{sec:conclusion}
We have presented a general variational framework for solving marginal MAP
problems approximately, opening new doors for developing efficient algorithms. 
In particular, we show that our proposed ``mixed-product" BP admits appealing theoretical properties and performs well in practice.

Potential future directions include improving the performance of the truncated TRW approximation
by optimizing weights, deriving optimality conditions that may be applicable
even when the sum component does not form a tree, studying the convergent properties of mixed-product BP, and 
leveraging our results to learn hidden variable models for data. 


\subsection*{Acknowledgments} 
We thank Arthur Choi for providing help on SamIam. 
This work was supported in part by the 
National Science Foundation (awards IIS-1065618 and IIS-1254071), 
and a Microsoft Research Ph.D Fellowship.

\bibliography{marginalMAP_jmlr} 

\newpage
\appendix
\newcommand{\RNum}[1]{\uppercase\expandafter{\romannumeral #1\relax}}

\qiangold{
\section{Proof of Proposition~\ref{pro:generalweightedBP}} \label{app:generalweightedBP}
\begin{proof}
The Lagrangian of \eqref{equ:generalF} with the local consistency constraint of $\localpoly$ in \eqref{equ:localpolydefine} is 
$$
\langle \vtheta, \vtau \rangle + \sum_{i\in V} [ w_i H_i(\vtau)  + \lambda_i^0 \sum_{x_i} \tau_i(x_i) ] - \sum_{(ij) \in E} [ w_{ij} I_{ij}(\vtau)  + \sum_{x_j} \lambda_{i\to j} (x_j) \sum_{x_i} ( \tau_{ij}(x_i, x_j ) - \tau_j (x_j) ) ]. 
$$
where $\{\lambda_i^0 \colon i\in V \}$ and $\{ \lambda_{j\to i}(x_i) \colon (ij)\in E, x_i \in \X_i \}$ are the Lagrange multipliers.  
Recall that
\begin{align*}
& \langle \vtheta, \vtau \rangle = \sum_{i\in V} \theta_i(x_i) \tau_i(x_i) + \sum_{(ij)\in E} \theta_{ij}(x_i, x_j) \tau_{ij}(x_i, x_j) , \\
& H_i(\vtau) = - \sum_{x_i} \tau_i(x_i) \log \tau_i(x_i) , \\ 
& I_{ij}(\vtau) = \sum_{x_i, x_j} \tau_{ij}(x_i, x_j) \log  \frac{\tau_{ij}(x_{i}, x_j)}{\sum_{x_i}\tau_{ij}(x_i, x_j)  \sum_{x_j}  \tau_{ij}(x_i, x_j) }. 
\end{align*}
Taking the derivative of the Lagrangian w.r.t. $\tau_i(x_i)$ and $\tau_{ij}(x_i, x_j)$, we have 
\begin{align}
&\theta_i(x_i)  - w_i \log \tau_i(x_i)    + \sum_{j\in \partial_i}  \lambda_{j\to i}(x_i)  = const, \label{equ:logweighted_marginals_1}  \\
&\theta_{ij}(x_i, x_j)  - w_{ij} \log \frac{\tau_{ij}(x_i, x_j)}{ \tau_i(x_i) \tau_j(x_j)}  + \lambda_{i\to j}(x_j) + \lambda_{j\to i}(x_i)  =  const,  \label{equ:logweighted_marginals_2} 
\end{align}
where we used the local consistency condition that $\sum_{x_j} \tau_{ij}(x_i, x_j) = \tau_i(x_i)$.  By defining $m_{i\to j}(x_j) = \exp(\lambda_{i\to j}(x_j))$, we obtain \eqref{equ:weighted_marginals} directly from \eqref{equ:logweighted_marginals_1}-\eqref{equ:logweighted_marginals_2}.\\
Plugging \eqref{equ:weighted_marginals} into the constraint that   $\sum_{x_j} \tau_{ij}(x_i, x_j) = \tau_i(x_i)$ gives \eqref{equ:weightedmsg}.
\end{proof}
}

\section{Proof of Theorem~\ref{thm:betheglobalopt}} 
\begin{proof} 
(i). For $\vtau \in \M^o $, the objective function in \eqref{equ:Phitree} equals
\begin{align}
F_{tree}(\vtau, \vtheta) 
& = \langle \vtheta, \vtau \rangle  \  +  \ \sum_{i\in V} \Hitau - \!\!\!\sum_{(ij)\in E_A} \!\!\! \Iijtau - \!\!\!\sum_{(ij)\in \cross} \!\!\!\! \rho_{ij} \Iijtau   \notag  \\
& = \langle \vtheta, \vtau \rangle  \  +  \ \sum_{i\in V} \Hitau - \!\!\!\sum_{(ij)\in E_A} \!\!\! \Iijtau   \label{equ:dd1}  \\
& = \langle \vtheta, \vtau \rangle  \  +  H_{A|B}(\vtau)    \label{equ:dd2}  \\
& = F_{mix}(\vtau, \vtheta),  \notag
\end{align}
where the equality in \eqref{equ:dd1} is because $\Iijtau = 0$ if $\forall (ij)\in \partial_{AB}$, 
and the equality in \eqref{equ:dd2} is because the sum part $G_A$ is a tree and we have the tree decomposition $H_{A|B} = \sum_{i\in V} \Hitau - \sum_{(ij) \in E_A} \Iijtau$. 
Therefore we have 
\begin{align}
\Phi_{tree}(\vtheta) = \max_{\vtau \in \localpoly}  F_{tree}(\vtau, \vtheta)  
\geq  \max_{\vtau \in \M^o}  F_{tree}(\vtau, \vtheta)    =   \max_{\vtau \in \M^o}  F_{mix}(\vtau, \vtheta)   = \Phi_{AB}(\vtheta),   \label{equ:dd3}
\end{align}
where the inequality is because $\M^o \subset \M \subset \localpoly $. 

If there exists $\vx_B^*$ such that $Q(\vx_B^* ; \vtheta) = \Phi_{tree}(\vtheta)$, then we have
$$
Q(\vx_B^* ; \vtheta)   =  \Phi_{tree}(\vtheta) \geq   \Phi_{AB}(\vtheta) = \max_{\vx_B} Q(\vx_B ; \vtheta).
$$ 
This proves that $\vx_B^*$ is a globally optimal marginal MAP solution. 

(ii). Because $\tau^*_i(x_i)$ for $\forall i \in B$ are deterministic, and the sum part $G_A$ is a tree,  we have that $\vtau^* \in \M^o $. Therefore the inequality in \eqref{equ:dd3} is tight, and we can conclude the proof by using Corollary~\ref{cor:nonconvex_mixduality}. 
\end{proof}

\section{Proof of Theorem~\ref{thm:localopt}} 
\begin{proof}
\newcommand{\crossA}{{\partial_A}}
By Theorem~\ref{thm:reparameter}, the beliefs $\{b_i, b_{ij}\}$ should satisfy the reparameterization property in \eqref{equ:repara} and the consistency conditions in \eqref{equ:sum_consistency}-\eqref{equ:mix_consistency}. 
Without loss of generality, we assume $\{b_i, b_{ij}\}$ are normalized such that $\sum_{x_i} b_i(x_i) = 1$ for $i \in A$ and $\max_{x_i} b_i(x_i) = 1$ for $i \in B$.

\RNum{1})  For simplicity, we first prove the case of $C = B$, when $G = G_{C\cup A}$ itself is a semi $A$-$B$ tree, and the theorem implies that $\vx_B^*$ is a global optimum. By the reparameterization condition, we have 
\begin{align}
p(\vx ) =  \hat{p}_B(\vx_B) \hat{p}_{A|B}(\vx), 
\label{equ:pabab}
\end{align}
where 
\begin{align}
&\hat{p}_B(\vx_B) = \prod_{i\in B} b_i (x_i)  \prod_{(ij) \in E_B} \bigg [ \frac{b_{ij}(x_i, x_j)}{b_i(x_i) b_j(x_j)} \bigg ]^{\rho_{ij}},  \label{equ:tmp1}\\ 
&\hat{p}_{A|B}(\vx) = \prod_{i\in A} b_i (x_i) \prod_{(ij) \in E_A }  \bigg [ \frac{b_{ij}(x_i, x_j)}{b_i(x_i) b_j(x_j)} \bigg ]^{\rho_{ij}}  \prod_{ (ij)\in \cross} \bigg [ \frac{b_{ij}(x_i, x_j)}{b_i(x_i) b_j(x_j)} \bigg ] ^{\rho_{ij}} .
\label{equ:tmp2}
\end{align}
Note we have 
$$
p(\vx_B) = \sum_{\vx_A} p(\vx) = \sum_{\vx_A} \hat{p}_{B}(\vx_B)  \hat{p}_{A|B}(\vx) = \hat{p}_B(\vx_B) \sum_{\vx_A} \hat{p}_{A|B} (\vx).  
$$
We just need to show that $\vx_B^*$ maximizes $\hat{p}_B(\vx_B)$ and $ \sum_{\vx_A} \hat{p}_{A|B} (\vx) $, respectively. 

First, since $\hat{p}_B(\vx_B)$ involves only the max nodes, a standard MAP analysis applies. Because the max part of the beliefs,  $\{b_i, b_{ij} \colon (ij)\in E_B \}$, satisfy the standard max-consistency conditions, and the corresponding TRW weights $\{\rho_{ij} \colon (ij) \in E_B\}$ are provably convex by assumption, we establish that $\vx_B^*$ is the MAP solution of $\hat{p}_B(\vx_B)$ by Theorem 1 of  \citet{Weiss07}. 

Secondly, to show that $\vx_B^*$ also maximizes $\hat{p}_{A|B} (\vx)$ requires the combination of the mixed-consistency and sum-consistency conditions. 
\newcommand{\pai}{{\mathrm{\pi}_i}}
Since $G$ is a semi $A$-$B$ tree, we denote by $\pai$ the unique parent node of $i$ ($\pai  = \emptyset$ if $i$ is a root). 
In addition, let $\crossA$ be the subset of $A$ whose parent nodes are in $B$, that is, $\crossA = \{i \in A \colon \pai \in B\}$. Equation \eqref{equ:tmp2} can be reformed into
\begin{align}
\hat{p}_{A|B}(\vx)=\prod_{i \in A \setminus  \crossA} \frac{b_{i, \pai}(x_i, x_{\pai})}{b_{\pai}(x_{\pai})}  \prod_{i \in \crossA}  \bigg [ \frac{b_{i, \pai}(x_i, x_{\pai})}{b_{\pai}(x_{\pai})} \bigg ]^{\rho_{i, \pai}}  \bigg [ b_i(x_i) \bigg ]^{1 - \rho_{i,\pai}}, 
\end{align}
where we used the fact that $\rho_{ij} = 1$ for $(ij) \in E_A$. 
Therefore, we have for any $\vx_B \in \X_B$, 
\begin{align}
\sum_{\vx_A} \hat{p}_{A|B}( \vx)  & = \sum_{\vx_A}  \bigg \{ \prod_{i \in A \setminus  \crossA} \frac{b_{i, \pai}(x_i, x_{\pai})}{b_{\pai}(x_{\pai})}  \prod_{i \in \crossA}  \bigg [ \frac{b_{i, \pai}(x_i, x_{\pai})}{b_{\pai}(x_{\pai})} \bigg ]^{\rho_{i, \pai}}  \bigg [ b_i(x_i) \bigg ]^{1 - \rho_{i,\pai}}  \bigg \}  \notag \\
& = \prod_{i \in \crossA}    \sum_{x_i}   \bigg [ \frac{b_{i, \pai}(x_i, x_{\pai})}{b_{\pai}(x_{\pai})} \bigg ]^{\rho_{i, \pai}}  \bigg [ b_i(x_i) \bigg ]^{1 - \rho_{i,\pai}}    \label{equ:elimABtree_equ} \\ 
&  \leq  \prod_{i \in \crossA}    \bigg [ \sum_{x_i}  \frac{b_{i, \pai}(x_i, x_{\pai})}{b_{\pai}(x_{\pai})} \bigg ]^{\rho_{i, \pai}}   \bigg [ \sum_{x_i} b_i(x_i) \bigg ]^{1 - \rho_{i,\pai}}  \label{equ:elimABtree_holder} \\
&   = 1 ,  \label{equ:elimABtree_one}
\end{align} 
where the equality in \eqref{equ:elimABtree_equ} eliminates (by summation) all the interior nodes in $A$. The inequality in \eqref{equ:elimABtree_holder} follows from H\"{o}lder's inequality. Finally, the equality in \eqref{equ:elimABtree_one} holds because all the sum part of beliefs $\{b_i, b_{ij} \colon (ij) \in E_A\}$ satisfies the sum-consistency \eqref{equ:sum_consistency}. 
  
On the other hand, for any $(i, \pi_i) \in \cross$, because $x_{\pai}^* = \argmax_{x_\pai} b_{\pai}(x_{\pai})$, we have $b_{i, \pai}(x_i, x_{\pai}^*)  = b_i(x_i)$ by the mixed-consistency condition \eqref{equ:mix_consistency}. Therefore,
  \begin{align}
  \sum_{\vx_A} \hat{p}_{A|B}( [\vx_A, \vx_B^*]) 
& = \prod_{i \in \crossA}    \sum_{x_i}   \bigg [ \frac{b_{i, \pai}(x_i, x_{\pai}^*)}{b_{\pai}(x_{\pai}^*)} \bigg ]^{\rho_{i, \pai}}  \bigg [ b_i(x_i) \bigg ]^{1 - \rho_{i,\pai}}    \\ 
& = \prod_{i \in \crossA}    \bigg [ \frac{1}{ b_{\pai}(x_{\pai}^*)} \bigg ] ^{\rho_{i, \pai}}  \sum_{x_i}    {b_{i}(x_i)} \\
& = 1 . \label{equ:tmp_proof_2}
\end{align}
Combining \eqref{equ:elimABtree_one} and \eqref{equ:tmp_proof_2}, we have 
$\sum_{\vx_A} \hat{p}_{A|B}(\vx) \leq \sum_{\vx_A} \hat{p}_{A|B}([\vx_A, \vx_B^*]) = 1$ for any $\vx_B \in \X_B$, that is, $\vx_B^*$ maximizes $\sum_{\vx_A} \hat{p}_{A|B}(\vx)$. This finishes the proof for the case $C = B$.  

\vspace{1em}

\RNum{2}) In the case of $C \neq B$, let $D = B \setminus C$. We decompose $p(\vx)$ into
$$
p(\vx) = \hat{p}_B([\vx_C, \vx_D]) \hat{p}_{A|C}([\vx_A, \vx_C]) \hat{r}_{AD}([\vx_A, \vx_D])
$$
where $\hat{p}_B(\vx_B)$ and $\hat{p}_{A|B}(\vx)$ are defined similarly to \eqref{equ:tmp1} and \eqref{equ:tmp2},
\begin{align}
&\hat{p}_B(\vx_B) = \prod_{i\in B} b_i (x_i)  \prod_{(ij) \in E_B} \bigg [ \frac{b_{ij}(x_i, x_j)}{b_i(x_i) b_j(x_j)} \bigg ]^{\rho_{ij}},  \label{equ:tmp1b}\\ 
&\hat{p}_{A|C}([\vx_A, \vx_C]) = \prod_{i\in A} b_i (x_i) \prod_{(ij) \in E_A }  \bigg [ \frac{b_{ij}(x_i, x_j)}{b_i(x_i) b_j(x_j)} \bigg ]^{\rho_{ij}}  \prod_{(ij)\in \partial_{AC}} \bigg [ \frac{b_{ij}(x_i, x_j)}{b_i(x_i) b_j(x_j)} \bigg ] ^{\rho_{ij}} , 
\label{equ:tmp2b}
\end{align}
where $\pai$ is the parent node of $i$ in the semi $A$-$B$ tree $G_{A\cup C}$ and $\partial_{AC}$ is set of edges across $A$ and $C$, that is, $\partial_{AC} = \{ (ij) \in E \colon i \in A , j \in C \}$. 
The term $\hat{r}_{AD}(\vx)$ is defined as
\begin{align}
&\hat{r}_{AD}([\vx_A, \vx_D]) =
\prod_{(ij)\in \partial_{AD}} \bigg [ \frac{b_{ij}(x_i, x_j)}{b_i(x_i) b_j(x_j)} \bigg ] ^{\rho_{ij}}  , 
\label{equ:tmp3b}
\end{align}
where similarly $\partial_{AD}$ is the set of edges across $A$ and $D$. 

Because $x_j^* = \argmax_{x_j} b_j(x_j)$ for $j \in D$, we have $b_{ij}(x_i,x_j^*) = b_{i}(x_i)$ for $(ij)\in \partial_{AD}$, $j\in D$ by the mixed-consistency condition in \eqref{equ:mix_consistency}. Therefore, one can show that $\hat{r}_{AD}([\vx_A, \vx_D^*]) = 1$, and hence
\begin{equation*}
p([\vx_A, \vx_C, \vx_D^*]) = \hat{p}_B([\vx_C, \vx_D^*]) \hat{p}_{A|C}([\vx_A, \vx_C]) . 
\end{equation*}
The remainder of the proof is similar to that for the case $C = B$: by the analysis in \citet{Weiss07}, it follows that $\vx_C^*\in  \argmax_{\vx_C} p([\vx_C, \vx_D^*])$, and we have previously shown that $\vx_C^*  \in \argmax_{\vx_C} \sum_{\vx_A}\hat{p}_{A|C}([\vx_A, \vx_C])$. This establishes that $x_C^*$ maximizes
$$ \sum_{\vx_A}p([\vx_A, \vx_C, \vx_D^*])  =  p([\vx_C, \vx_D^*])  \sum_{\vx_A}\hat{p}_{A|C}([\vx_A, \vx_C]) , $$
 which concludes the proof.  
\end{proof}

\end{document}